\documentclass[]{article}

\usepackage{amsmath,amsfonts,amssymb,amsthm}
\usepackage{natbib}
\usepackage{graphicx}

\date{August 2010}

\newcommand{\bx}{\mathbf{x}}
\newcommand{\cX}{\mathcal{X}}

\newcommand{\bD}{\mathbf{D}}
\newcommand{\cD}{\mathcal{D}}

\newcommand{\st}{\mbox{s.t.}}
\newcommand{\blambda}{\boldsymbol{\lambda}}

\newcommand{\fixme}[1]{}

\newtheorem{lemma}{Lemma}
\newtheorem{theorem}{Theorem}
\newtheorem{definition}{Definition}

\begin{document}

\title{Maximum entropy models and subjective interestingness: an application to tiles in binary databases}

\author{Tijl De Bie\\ Intelligent Systems Laboratory, University of Bristol \\ \texttt{tijl.debie@gmail.com}}

\maketitle

\begin{abstract}
Recent research has highlighted the practical benefits of subjective interestingness measures, which quantify the novelty or unexpectedness of a pattern when contrasted with any prior information of the data miner \citep{SiT:95,GeH:06}. A key challenge here is the formalization of this prior information in a way that lends itself to the definition of an interestingness subjective measure that is both meaningful and practical.

In this paper, we outline a \emph{general strategy} of how this could be achieved, before \emph{working out the details for a use case} that is important in its own right.

Our general strategy is based on considering prior information as constraints on a probabilistic model representing the uncertainty about the data. More specifically, we represent the prior information 
by the maximum entropy (MaxEnt) distribution subject to these constraints. We briefly outline various measures that could subsequently be used to contrast patterns with this MaxEnt model, thus quantifying their subjective interestingness.

We demonstrate this strategy for rectangular databases with knowledge of the row and column sums. This situation has been considered before using computation intensive approaches based on swap randomizations, allowing for the computation of empirical p-values as interestingness measures \citep{swap07}. We show how the MaxEnt model can be computed remarkably efficiently in this situation, and how it can be used for the same purpose as swap randomizations but computationally more efficiently. More importantly, being an explicitly represented distribution, the MaxEnt model can additionally be used to define analytically computable interestingness measures, as we demonstrate for tiles \citep{tiles04} in binary databases.

\end{abstract}

\section{Introduction}

\subsection{Prior work on subjective interestingness}\label{prior_si}

\paragraph{Prior information and interestingness of patterns}
Data mining practitioners commonly have a partial understanding of the structure of the data investigated. The goal of the data mining process is then to discover any additional structure or patterns the data may exhibit. Unfortunately, structure that is trivially implied by the prior information available is often overwhelming, and it is hard to design data mining algorithms that look beyond it.

For example, it should not be seen as a surprise that items known to be frequent in a binary database are jointly part of many transactions, as this is what should be expected even under a model of independence. Rather than discovering such patterns that are implied by prior information, data mining is concerned with discovering departures from this prior information.

Interestingness measures that take into account prior information in this way are commonly referred to as subjective interestingness measures, first introduced as a concept in \citet{SiT:95}. In contrast with objective interestingness measures (such as the support of an itemset and the confidence of an association rule), they do not depend on the data alone but also on the prior information of the data miner. An excellent overview of subjective and objective interestingness measures for data mining can be found in \citet{GeH:06}.

To define subjective interestingness measures, the ability to formalize prior information is as important as the ability to contrast patterns with this information thus formalized. In this paper, we mainly focus on the first of these challenges: the task of designing appropriate models incorporating prior information in data mining contexts. However, we will also outline various possible approaches of how such a background model can be used to define subjective interestingness measures, and we will demonstrate one in greater detail on a practical use case.

\paragraph{Prior work on subjective interestingness measures}
Several authors have already suggested ways to incorporate prior information in the data mining process for this purpose of defining subjective interestingness measures.

In \citet{SiT:95}, which introduces the idea of subjective interestingness measures, and in later work \citep[e.g.][]{PaT:98,PaT:00}, prior information is formalized as a set of beliefs, each of which holds with a certain confidence. The beliefs they consider are of the form of rules $X\rightarrow Y$ where $X$ and $Y$ are conjunctions of literals. Patterns in the form $A\rightarrow B$ are then assessed for unexpectedness in a well-defined way with respect to each belief. A disadvantage of this approach is that it is local in that each belief is treated independently of the others. Furthermore, it is specifically designed for patterns in the form of rules.

An approach that overcomes these problems was proposed in \citet{Jas:04}, still for binary databases. They propose to use a Bayesian network model for the transactions to formalize background knowledge. They then use this model to compute the difference between the expected frequency of an itemset and its observed frequency in the data as a subjective measure of interestingness.

Despite its potential, this approach suffers from a few limitations. First, it may not always be clear how the Bayesian network needs to be designed to accord with the prior information. The approach is particularly impractical for data mining practitioners unfamiliar with Bayesian networks. Second, it treats transactions as i.i.d. random variables. And third and probably most seriously, the variables in the Bayesian network are the items (or attribute-values), such that prior information on individual transactions cannot be taken into account.

An approach that resolves all three these problems, albeit for particular types of data and prior information, is presented in \citet{swap07}. In this work, the authors show how one can assess the significance of data mining results in binary databases with respect to prior information on the row and column sums. Their methodology relies on swap randomizations, which leave the row and column sums invariant. By iteratively applying swap randomizations they show how one can approximately sample from the uniform distribution over all databases with row and column sums as specified by the prior information. This can be used by computationally intensive approaches \citep[e.g.][]{compstat05} for estimating the significance of data mining results as quantified by the empirical p-value. Later this work was extended to real-valued data \citep{OVK:08} and to more complex constraints besides row and column sums \citep{HOV:09}.

The statistical assessment of data mining results using the randomization methods from \citet{swap07}, \citet{OVK:08}, \citet{Man:08}, and \citet{HOV:09} is extremely useful and deserves a central place in data mining practice. However, it would be even more useful if a model for prior information could be used to directly guide data mining algorithms toward the subjectively more interesting patterns. Unfortunately, from a practical point of view, the use of models that are defined implicitly in terms of invariants seems limited to \emph{post-hoc} analyses. Indeed, it seems hard to scale algorithms that need to explore an entire search space of possible patterns if they need to assess each candidate by means of a randomization test or by referring to a large number of randomized data sets. Thus it is unclear if and how they could be used to define practical measures of interestingness other than empirical p-values. A further and at least as serious disadvantage of randomization methods is that their resolution is limited by the inverse of the number of randomized data sets considered. This is a problem in the highly relevant region of small p-values where a high resolution is important.

In contrast to this, an explicit analytical model capable of formalizing important types of prior information would enable one to assess patterns in an analytical way rather than in a computationally intensive way. Interestingness could then be quantified using exact hypothesis testing as in \citet{Gallo07,gallo09}, where a relatively simple independence model for items and transactions was used as a null model formalizing prior information. Alternatively, information theoretic principles could be applied to quantify the information content of a pattern, as done in \citet{Item06Siebes} for defining an objective interestingness measure, but then with respect to a background model defined by the prior information. We will argue that the results in this paper will make this possible.

\subsection{Contributions in this paper}

In this paper, we present a methodology for efficiently computing \emph{explicitly representable} probabilistic models for \emph{general types of data}, able to incorporate \emph{broad classes of prior information}. Our approach is based on the maximum entropy (MaxEnt) principle \citep{Jaynes82}. In Sec.~\ref{framework}, we first sketch the methodology in its full generality, and we briefly outline various ways in which such a MaxEnt model could be used to define subjective interestingness measures. This general framework is the first contribution in this paper.

In the second part of the paper we demonstrate this approach for rectangular databases with constraints on the row and column marginals as prior information, and for patterns in the form of tiles \citet{tiles04}. The purpose of this second part is twofold. First, this particular use case is important in its own right, and has received a significant amount of attention in the literature \cite[e.g.][]{swap07,OVK:08,HOV:09}. Second, we hope that elaborating on this use case may support and clarify the general approach outlined in Sec.~\ref{framework}, thus underscoring its wider potential for the definition of subjective interestingness measures also in other situations.

This second part of the paper is structured as follows. In Sec.~\ref{mainsection}, we derive the MaxEnt model for rectangular databases under row and column sum constraints, and show how it can be computed remarkably efficiently. We do this for binary, positive integer-valued and positive real-valued databases as it comes at virtually no extra cost as compared to just dealing with binary data. In Sec.~\ref{invariancesection}, we relate these MaxEnt models to distributions defined implicitly by swap randomizations. In particular, we prove invariance of these MaxEnt models to a generalized type of swap randomization. In Sec.~\ref{use}, we show that it is computationally cheap to sample randomized databases from the MaxEnt model, such that it can be used as an efficient alternative to swap randomizations. More importantly, we show how it can be used to define subjective interestingness of tile patterns with respect to prior information on the row and column sums of the database. In  Sec.~\ref{discussion} we point out some interesting relations with literature. And in Sec.~\ref{experiments} we provide experiments demonstrating the efficiency and scalability as well as experiments to assess usefulness of the subjective interestingness measure for tiles.

This paper significantly extends two unpublished technical reports \citep{DeB:09a,DeB:09b}, the first one about the maximum entropy modeling approach, the second one primarily introducing a subjective interestingness measure of which the one in the present paper is a refinement.

\section{Formalizing prior information, and subjective interestingness measures: a general approach}\label{framework}

In this Section, we introduce the MaxEnt modeling strategy at a general level (Sec.~\ref{maxent}), and outline ways in which such a probabilistic representation for prior information could be used to define subjective interestingness measures for patterns (Sec.~\ref{maxent_for_interestingness}).

We should stress that in this Section, we have no intention to be overly specific or focused on a particular type of data, prior information, or pattern type. Instead, our goal is to outline some general principles and ideas, centred around the formalization of prior information in a MaxEnt model. To become practical, these ideas need be developed and specified further in additional research, and we demonstrate this for a particular use case in the later Sections in this paper.

\subsection{The maximum entropy principle to model prior information}\label{maxent}

Here we will introduce the maximum entropy principle and its use for modeling prior information in full generality. In Sec.~\ref{mainsection} we will then apply this to the special case of rectangular databases and prior information on the row and column sums.

Let $\cX$ be any countable set.\footnote{The results below can easily be extended for measurable sets as well, but we present them for countable sets for notational simplicity.} Consider the problem of finding a probability distribution $P$ over
the data $\bx\in\cX$ that satisfies a set of linear constraints implied by prior information. In particular, we will consider constraints of the form:
\begin{eqnarray}
\sum_\bx P(\bx)f_i(\bx)&=&d_i,\label{constraints_f}
\end{eqnarray}
where $f_i$ are real-valued functions of the data. Regarding these functions $f_i$ as properties of the data, these constraints could be the formalization of certain `expectations' (in both the formal and informal meaning of the word) of a data miner about these properties in the data. As mentioned earlier, we will give a specific example in Sec.~\ref{mainsection}, but for now we will focus on the implications of such a set of constraints on the shape of the probability distribution. 

In general, these constraints will not be sufficient to uniquely
determine the distribution of the data. A common strategy to
overcome this problem is to search for the distribution that has the
largest entropy subject to these constraints, to which we will refer
as the MaxEnt distribution. Mathematically, it is found as the
solution of:
\begin{eqnarray}
&\max_{P(\bx)}& -\sum_\bx P(\bx)\log P(\bx),\label{maxentobjective}\\
&\st& \sum_\bx P(\bx)f_i(\bx)=d_i,\ (\forall i)\label{ci_general}\\
&& \sum_\bx P(\bx)=1,\label{c1_general}
\end{eqnarray}
where the last constraint ensures that $P(\bx)$ is properly normalized.

Originally advocated in \citet{Jaynes57,Jaynes82} as a
generalization of Laplace's principle of indifference, the choice for the MaxEnt
distribution can be defended in a variety of ways. The most common
argument is that any distribution other than the MaxEnt distribution
effectively makes additional assumptions about the data that reduce
the entropy. As making additional assumptions biases the
distribution in undue ways, the MaxEnt distribution is the safest
bet.

A lesser known argument, but not less convincing, is a
game-theoretic one \citep{Top:79}. Assuming that the true data
distribution satisfies the given constraints, it remarks that the
Shannon-optimal compression code (e.g. Huffman) designed based on the MaxEnt
distribution minimizes the worst-case expected coding length of a
message coming from the true distribution. Hence, using the MaxEnt
distribution for designing a code is optimal in a robust minimax
sense.

Besides these motivations for the MaxEnt principle, it is also
relatively easy to compute a MaxEnt model. Indeed, the MaxEnt optimization problem (\ref{maxentobjective}-\ref{c1_general}) is convex, and can be solved using standard techniques from convex optimization theory \citep{boyd04}. Let us use Lagrange
multiplier $\mu$ for constraint~(\ref{c1_general}) and $\lambda_i$
for constraints~(\ref{ci_general}). Using $\blambda$ to denote the vector containing all Lagrange multipliers $\lambda_i$, the Lagrangian is then equal to:
\begin{eqnarray}
L(\mu,\blambda,P(\bx))&=&-\sum_\bx P(\bx)\log P(\bx)\label{lagrangian}\\
&&+\sum_i\lambda_i\left(\sum_\bx P(\bx)f_i(\bx)-d_i\right)+\mu\left(\sum_\bx P(\bx)-1\right).\nonumber
\end{eqnarray}
Equating the derivative w.r.t. $P(\bx)$ to $0$ yields the
optimality conditions:
\begin{eqnarray*}
\log P(\bx) &=& \mu-1+\sum_i \lambda_i f_i(\bx),\\
\Leftrightarrow P(\bx)&=&\frac{1}{Z}\exp\left(\sum_i\lambda_i f_i(\bx)\right),
\end{eqnarray*}
where we introduced a new variable $Z=\exp(1-\mu)$. The normalization constraint $\sum_\bx P(\bx)=1$ is often imposed constructively by choosing $Z$ to be an appropriate function of the other Lagrange multipliers $\blambda$, in particular: \begin{eqnarray}
Z(\blambda)&=&\sum_\bx\exp\left(\sum_i\lambda_i f_i(\bx)\right).\label{partition_function}
\end{eqnarray}
The function $Z(\blambda)$ is known as the partition function. This leads to the final form of the MaxEnt distribution as a function of the Lagrange multipliers $\blambda$:
\begin{eqnarray}\label{exponential_family}
P(\bx)&=&\frac{1}{Z(\blambda)}\exp\left(\sum_i\lambda_i f_i(\bx)\right).
\end{eqnarray}

The resulting model is a member of the exponential family of distributions, such that all existing theory for this family of distributions can be used \citep[e.g.][]{WaJ:08}. The optimal values of the Lagrange multipliers $\blambda$ can be found by minimizing the Lagrange dual objective. This Lagrange dual is obtained by substituting Eq.~(\ref{exponential_family}) for $P(\bx)$ in the Lagrangian  (Eq.~(\ref{lagrangian})). After some algebra:
\begin{eqnarray}
L(\blambda)&=&\log(Z(\blambda))-\sum_i\lambda_i d_i.\label{lagrange_dual}
\end{eqnarray}
Minimizing $L(\blambda)$ thus yields the values for the Lagrange multipliers and thus the MaxEnt distribution. In passing we note that it is easy to see that $L(\blambda)$ is equal to the negative log-likelihood of distribution from Eq.~(\ref{exponential_family}) on data $\bx$ that satisfies $f_i(\bx)=d_i$. Hence, the MaxEnt distribution subject to constraints~(\ref{ci_general}) is equivalent to the maximum likelihood distribution of the form~(\ref{exponential_family}) fitted to data for which the constraints $f_i(\bx)=d_i$ hold deterministically rather than in expectation  \citep[e.g.][]{WaJ:08}.

The log-partition function $\log(Z(\blambda))$ is well-known to be a convex function, such that $L(\blambda)$ is convex as well. Thanks to this, it turns out that minimizing $L(\blambda)$ can be done efficiently for a broad class of constraints, using standard techniques for convex optimization (see Sec.~\ref{sec_lagrange}) or using special purpose techniques such as Iterative Proportional Fitting \citep[e.g.][]{WaJ:08}.

A full discussion of the efficiency of the optimization of the Lagrange multipliers in the most general form of the problem is beyond the scope of this paper. Let us just point out that many results from the graphical models literature can be borrowed to establish tractability results. Rather than staying at the general level, we choose to fully explore a specific type of data and prior information of particular interest to data mining in Sec.~\ref{mainsection} below.

\subsection{Using the MaxEnt model to define subjective interestingness measures}\label{maxent_for_interestingness}

Given a representation of prior information in the form the MaxEnt distribution, subjective interestingness of a pattern can be quantified by contrasting it with the MaxEnt distribution. This can be done by computing some measure of unexpectedness of the pattern e.g. using hypothesis testing, or by relying on information theory. Here is a non-exhaustive list of possibilities:
\begin{enumerate}
\item \emph{Self-information.} A first option for quantifying interestingness of a pattern relies on the probability of the pattern under MaxEnt model. The smaller this probability is, the more surprising the pattern when contrasted with the prior information. Equivalently, the negative log-probability can be used, which represents the coding length required to encode the pattern with a Shannon optimal code with respect to the MaxEnt model. In Shannon's information theory, such a negative log-probability is known as the self-information \citep[e.g.][]{CoT:92}. The larger this quantity, the more information the pattern contains. Interestingly, in thermodynamics the self-information is also known as \emph{surprisal} \citep{Tri:61}.
\item \emph{Information compression ratio.} The self-information of a pattern does not take into account the complexity of \emph{describing} or \emph{communicating} the pattern to the data miner. This complexity could be formalized as the description length of the pattern in a code that assigns longer code-words to patterns that are perceived as more complex by the data miner. Then, we can define a subjective interestingness measure as the ratio of the self-information of the pattern given the MaxEnt model and the pattern's description length. This would correspond to some kind of compression ratio: how much information is compressed in the description of a pattern?
\item \emph{P-value} A third option is the probability of the pattern or a stronger instantiation of the pattern to be present in the data, with respect to the MaxEnt model as null hypothesis. This probability is known as a p-value in statistics \citep[e.g.][]{LeR:05}, and computing the p-value is at the core of hypothesis testing. With the MaxEnt model as null hypothesis, patterns with a small p-value are then those that are maximally surprising given the prior information embedded in it. Hence, in a certain well-defined sense these will be unexpected to the data miner.
\item \emph{P-value based on the likelihood ratio test.} To use the previous approach, a notion of \emph{pattern strength} needs to be chosen as a test statistic, such as the frequency of an itemset. Perhaps a more principled approach is by relying on the ratio of the probability of the data under the MaxEnt model, versus the probability of the data under an augmented model that is corrected for the presence of the found pattern. The augmented model can also be found using MaxEnt, with an additional constraint for the fact that the pattern is there. Based on this \emph{likelihood ratio}, a p-value can be computed using a likelihood ratio test, if certain regularity conditions are satisfied \citep[e.g.][]{LeR:05}.
\end{enumerate}

It seems hard to discuss any of these approaches more formally without being more specific about the particular type of pattern concerned. For this reason, in Sec.~\ref{use_case} we will demonstrate the second option (Information compression ratio), which we regard as particularly promising, for the particular of tiles in binary databases. Working out the details of the other approaches, and the connections between them, could be the subject of further research on the topic of subjective interestingness.

\section{MaxEnt Distributions for Rectangular Databases}\label{mainsection}

In the rest of the paper, we will elaborate on the details of the outlined general approach for specific types of data, prior information, and patterns. In the current Section, we will apply the general MaxEnt modeling strategy to the important case of rectangular databases. To this end, we will cast the prior information in the general form of Eq.~(\ref{constraints_f}). We will investigate the specific form of the resulting MaxEnt model, and show how it can be fitted in a remarkably efficient way.

\subsection{Notation}

In the rest of this paper, we will denote the database using the matrix $\bD$ with $m$ rows and $n$ columns. To maintain generality, we will assume that all matrix values belong to some specified set $\cD\subseteq\mathbb{R}^+$, i.e. $\bD(i,j)\in\cD$. Later we will choose the set $\cD$ to be the set $\{0,1\}$ (to model binary databases), the set of positive integers (to model integer-valued databases), or the set of positive reals (to model real-valued databases). Other choices can be made, and it is fairly straightforward to adapt the derivations accordingly. For notational simplicity, in the subsequent derivations we will assume that $\cD$ is discrete and countable. However, if $\cD$ is continuous the derivations can be adapted easily.

\subsection{Swap randomizations and prior information for databases}\label{priorinf}

For binary databases, it has been argued that row and column sums can often be assumed as prior information.\footnote{We refer to \citet{swap07} for a detailed argumentation, and to the experiments in Sec.~\ref{exp_tiles} of this paper for a particular use case} Any pattern that can be explained by referring to row or column sums in a binary database is then deemed uninformative. Previous work has introduced ways to assess data mining results based on this assumption \citep{swap07,Man:08,HOV:09}. These methods rely on the ability to sample random databases from the uniform distribution over all databases that satisfy the prior information. To assess a frequent itemset in the given database, they then compute the empirical p-value as a subjective interestingness measure, defined as the fraction of the random databases in which the itemset is at least as frequent as in the given database.

Unfortunately, sampling from this uniform distribution cannot be done in a direct way. To overcome this, the authors randomize the given database by iteratively applying elementary randomization operations: so-called swap randomizations that transform any $2\times 2$ submatrix of the form $\left(\begin{array}{cc}1&0\\0&1\end{array}\right)$ into $\left(\begin{array}{cc}0&1\\1&0\end{array}\right)$. (See Fig.~\ref{swaps} for a graphical illustration.) Clearly, such operations leave the row and column sums invariant. Furthermore, \citet{swap07} showed how the limit distribution of a Markov chain of random swap operations is equal to the uniform distribution over all databases with the specified row and column marginals. Hence, one can approximately sample from this distribution by running this Markov chain for a sufficiently long time (although there are no theoretical results on convergence rates). The swap operation has later been generalized to deal with real-valued databases as well \citep{OVK:08}, and we will get back to this in Sec.~\ref{invariancesection}.

\begin{figure}
\begin{center}
\includegraphics[width=0.7\columnwidth, angle=270]{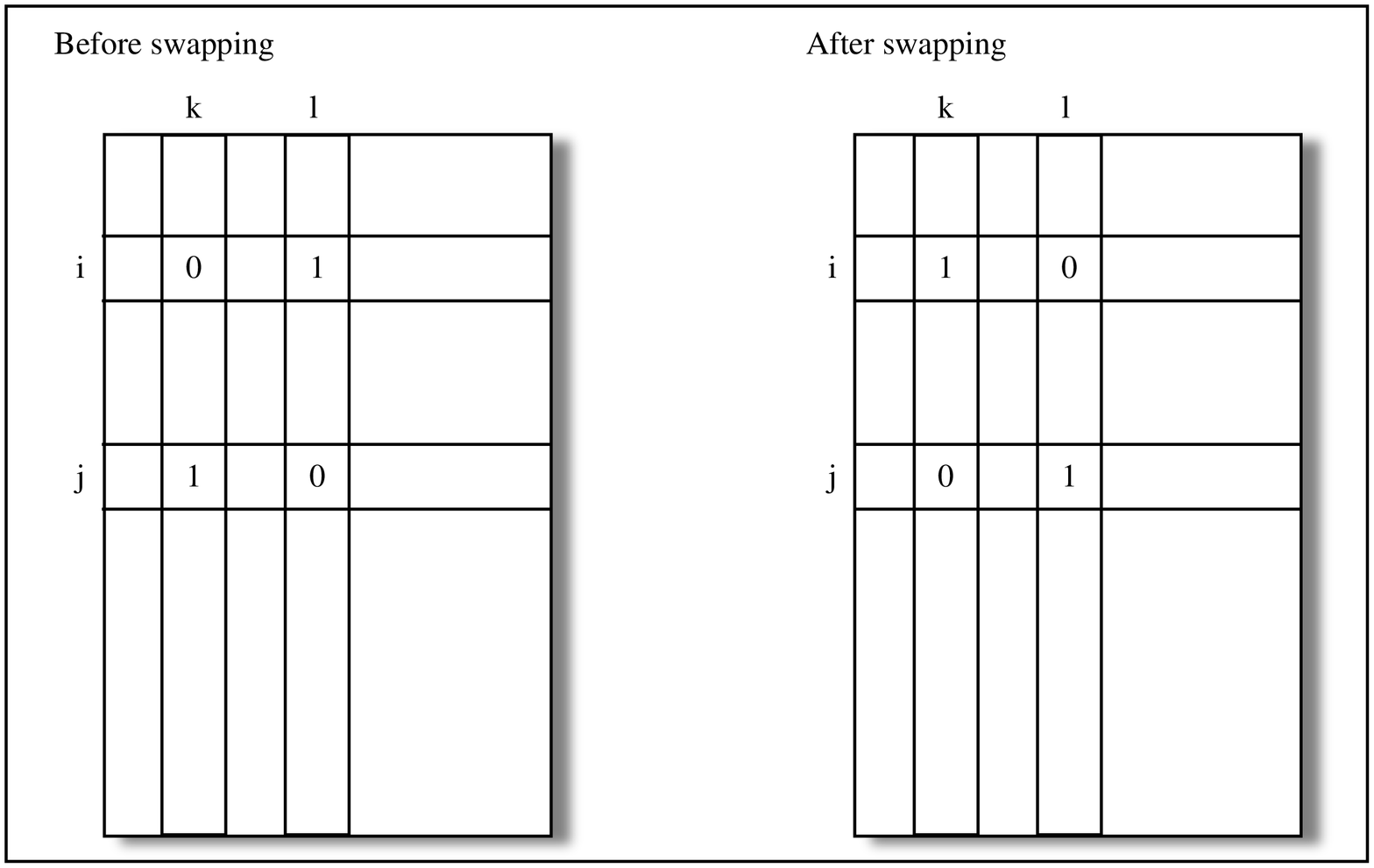}
\caption{The effect of a swap operation to a binary database.}\label{swaps}
\end{center}
\end{figure}

The models we will develop in this paper are based on exactly these invariants of the row and column sums, be it in a somewhat relaxed form: we will assume that the \emph{expected values} of the row and column sums are equal to specified values.
Mathematically, this can be expressed as:
\begin{eqnarray*}
\sum_{\bD\in\cD^{m\times n}} P(\bD)\left(\sum_j \bD(i,j)\right)&=&d^r_i,\\
\sum_{\bD\in\cD^{m\times n}} P(\bD)\left(\sum_i
\bD(i,j)\right)&=&d^c_j,
\end{eqnarray*}
where $d_i^r$ is the $i$'th expected row sum and $d_j^c$ the $j$'th
expected column sum. Although they have been suggested for binary
databases \citep{swap07}, and
later extended to real-valued databases \citep{OVK:08}, we will
explore the consequences of these constraints in broader generality,
for various choices for $\mathcal{D}$.

Importantly, it is easy to verify that these constraints are exactly
of the type of Eq.~(\ref{constraints_f}), such that the MaxEnt
formalism is directly applicable.

\subsection{MaxEnt matrix distributions with given expected row and column sums}

The MaxEnt distribution over the set of $m\times n$ matrices $\bD$ subject to constraints on the expected row and column sums is thus found by solving:
\begin{eqnarray}
&\max_{P(\bD)}& -\sum_{\bD} P(\bD)\log(P(\bD)),\nonumber\\
&\st& \sum_{\bD}
P(\bD)\left(\sum_j \bD(i,j)\right)=d^r_i,\label{ci}\\
&& \sum_{\bD} P(\bD)\left(\sum_i
\bD(i,j)\right)=d^c_j,\label{cj}\\
&& \sum_{\bD} P(\bD)=1.\label{c1}
\end{eqnarray}

As shown in Sec.~\ref{maxent}, the resulting distribution will belong to the exponential family, and will be of the form of Eq.~(\ref{exponential_family}). Using Lagrange multipliers $\lambda_i^r$ for constraints~(\ref{ci}) and $\lambda_j^c$ for constraints~(\ref{cj}) this yields:
\begin{eqnarray}
P(\bD)&=&\frac{1}{Z(\blambda^r,\blambda^c)}\exp\Bigg[\sum_i \lambda_i^r \left(\sum_j \bD(i,j)\right)+\sum_j \lambda_j^c \left(\sum_i \bD(i,j)\right)\Bigg],\label{PD1first}\\
&=&\frac{1}{Z(\blambda^r,\blambda^c)}\exp\Bigg[\sum_{i,j}\bD(i,j)(\lambda_i^r+\lambda_j^c)\Bigg],\nonumber\\
&=&\frac{1}{Z(\blambda^r,\blambda^c)}\prod_{i,j}\exp\left(\bD(i,j)(\lambda^r_i+\lambda^c_j)\right),\label{PD1}
\end{eqnarray}
where $Z(\blambda^r,\blambda^c)$ is the partition function,
$\blambda^r$ is the vector of Lagrange multipliers $\lambda_i^r$ for constraints
(\ref{ci}), and $\blambda^c$ the vector of Lagrange multipliers $\lambda_j^c$ for constraints (\ref{cj}).

Let us have a more detailed look at the partition function. Following Eq.~(\ref{partition_function}), it is equal to:
\begin{eqnarray*}
Z(\blambda^r,\blambda^c) &=&
\sum_{\bD\in\cD^{m\times n}}\prod_{i,j}\exp\left(\bD(i,j)(\lambda^r_i+\lambda^c_j)\right),\\
&=& \prod_{i,j}\sum_{\bD(i,j)\in\cD}\exp\left(\bD(i,j)(\lambda^r_i+\lambda^c_j)\right),\\
&=& \prod_{i,j}Z(\lambda^r_i,\lambda^c_j),
\end{eqnarray*}
where $Z(\lambda^r_i,\lambda^c_j)=
\sum_{\bD(i,j)\in\cD}\exp\left(\bD(i,j)(\lambda^r_i+\lambda^c_j)\right)$.

Plugging this into Eq.~(\ref{PD1}) yields:
\begin{eqnarray*}
P(\bD) &=&
\prod_{i,j}\frac{1}{Z(\lambda^r_i,\lambda^c_j)}\exp\left(\bD(i,j)(\lambda^r_i+\lambda^c_j)\right).
\end{eqnarray*}
This is a product of exponential family distributions for each of the elements in the matrix $\bD$. The partition function $Z(\blambda^r,\blambda^c)$ is the product of the partition functions $Z(\lambda_i^r,\lambda_j^c)$ of each of these individual distributions. Thus, we have proved the following Theorem.
\begin{theorem}
The MaxEnt distribution for matrices $\bD\in\cD^{m\times n}$ subject to constraints on the expected row and column sums is of the form:
\begin{eqnarray}
P(\bD) &=& \prod_{i,j}P_{ij}(\bD(i,j)),\label{eq_independent}
\end{eqnarray}
where
\begin{eqnarray}
P_{ij}(\bD(i,j)) =
\frac{1}{Z(\lambda^r_i,\lambda^c_j)}\exp\left(\bD(i,j)(\lambda^r_i+\lambda^c_j)\right)\label{Pindividual}
\end{eqnarray}
is a properly normalized probability distribution for the matrix element $\bD(i,j)$ at row $i$ and column $j$. Hence, the MaxEnt model factorizes as a product of independent distributions for the matrix elements.
\end{theorem}
It is important to stress that we did not impose
independence at the outset. The independence is a consequence of the
MaxEnt objective.

Various particular choices for $\cD$ will lead to various distributions, with
appropriate values for the normalization constant
$Z(\lambda^r_i,\lambda^c_j)$. Let us discuss three examples in detail.

\paragraph{Bernoulli} For $\cD=\{0,1\}$, the partition function \begin{eqnarray*}
Z(\lambda^r_i,\lambda^c_j)&=&1+\exp\left(\lambda^r_i+\lambda^c_j\right),
\end{eqnarray*}
such that:
\begin{eqnarray}
P_{ij}(\bD(i,j))&=&\left\{\begin{array}{ll}\frac{\exp\left(\lambda^r_i+\lambda^c_j\right)}{1+\exp\left(\lambda^r_i+\lambda^c_j\right)},&\mbox{ if } \bD(i,j)=1,\\
\frac{1}{1+\exp\left(\lambda^r_i+\lambda^c_j\right)},&\mbox{ if } \bD(i,j)=0.\end{array}\right.\label{Pindividual_Bernoulli}
\end{eqnarray}
This means that for $\cD=\{0,1\}$ the MaxEnt
distribution for $\bD$ reduces to a product of independent Bernoulli
distributions, with probability of success equal to
$p_{ij}=\frac{\exp{(\lambda_i^r+\lambda_j^c)}}{1+\exp{(\lambda_i^r+\lambda_j^c)}}$
for $\bD(i,j)$.

\paragraph{Geometric} For $\cD=\mathbb{N}$, the partition function \begin{eqnarray*}
Z(\lambda^r_i,\lambda^c_j)&=&
\sum_{k=0}^\infty \exp\left(k(\lambda^r_i+\lambda^c_j)\right),\\
&=&\frac{1}{1-\exp(\lambda^r_i+\lambda^c_j)},
\end{eqnarray*}
assuming that $\lambda^r_i+\lambda^c_j<0$ to ensure convergence of the sum.
Thus:
\begin{eqnarray*}
P_{ij}(\bD(i,j))&=&\left[1-\exp(\lambda^r_i+\lambda^c_j)\right]\cdot
\exp\left(\bD(i,j)(\lambda^r_i+\lambda^c_j)\right),\\
&=&\left[1-\exp(\lambda^r_i+\lambda^c_j)\right]\cdot
\left[\exp(\lambda^r_i+\lambda^c_j)\right]^{\bD(i,j)}.
\end{eqnarray*}
This means that for $\cD=\mathbb{N}$ the MaxEnt distribution for $\bD$ reduces to a product of independent geometric distributions, with probability of success equal to $1-\exp(\lambda^r_i+\lambda^c_j)$
for the matrix element $\bD(i,j)$.

\paragraph{Exponential} For $\cD=\mathbb{R}^+$, the partition function \begin{eqnarray*}
Z(\lambda^r_i,\lambda^c_j)&=&
\int_{0}^\infty \exp\left(x(\lambda^r_i+\lambda^c_j)\right)dx,\\
&=&-\frac{1}{\lambda^r_i+\lambda^c_j},
\end{eqnarray*}
assuming that $\lambda^r_i+\lambda^c_j<0$ to ensure convergence of the integral. Thus:
\begin{eqnarray*}
P_{ij}(\bD(i,j))&=&-(\lambda^r_i+\lambda^c_j)\cdot
\exp\left(\bD(i,j)(\lambda^r_i+\lambda^c_j)\right).
\end{eqnarray*}
This means that for $\cD=\mathbb{R}$ the MaxEnt distribution for $\bD$ reduces to a product of independent exponential distributions, with rate parameter equal to $-(\lambda^r_i+\lambda^c_j)$ for the matrix element $\bD(i,j)$.

These results are summarized in Table~\ref{partition_functions_table}.
\begin{table}
\centering \caption{Three possible domains for the elements of
$\bD$, the corresponding partition functions in the MaxEnt
distribution $P(\bD)$ for the matrix element $\bD(i,j)$, and the
resulting type of distribution for the matrix
elements.}\label{partition_functions_table}
\begin{tabular}{|c|c|c|rl|}\hline
 $\cD$&Distribution & $Z(\lambda^r_i,\lambda^c_j)$ & Parameter: & Value\\\hline
 $\{0,1\}$&Bernoulli & $1+\exp(\lambda^r_i+\lambda^c_j)$ & success prob.: & $\frac{\exp(\lambda^r_i+\lambda^c_j)}{1+\exp(\lambda^r_i+\lambda^c_j)}$\\
 $\mathbb{N}$&Geometric & $\frac{1}{1-\exp(\lambda^r_i+\lambda^c_j)}$ & success prob.: & $1-\exp(\lambda^r_i+\lambda^c_j)$\\
 $\mathbb{R}^+$&Exponential & $-\frac{1}{\lambda^r_i+\lambda^c_j}$ & rate param.: & $-(\lambda^r_i+\lambda^c_j)$\\\hline
\end{tabular}
\end{table}

\subsection{Optimizing the Lagrange multipliers}\label{sec_lagrange}

We have now derived the shape of the models $P(\bD)$, expressed in
terms of the Lagrange multipliers, but we have not yet discussed how to compute the values of these Lagrange multipliers at the MaxEnt optimum.

In Sec.~\ref{maxent}, we outlined the general strategy to do this: the optimal values for the Lagrange multipliers are found by minimizing the Lagrange dual as given by Eq.~(\ref{lagrange_dual}). For concreteness, let us go through the mathematical details for the case of a rectangular binary matrix:  $\bD\in\{0,1\}^{m\times n}$. There should be no conceptual difficulties in adapting the derivations below for other choices of $\cD$, and for conciseness these adaptations are omitted from this paper.

For $\bD\in\{0,1\}^{m\times n}$, the Lagrange dual from Eq.~(\ref{lagrange_dual}) is equal to:
\begin{eqnarray*}
L(\blambda^r,\blambda^c)&=&\log(Z(\blambda^r,\blambda^c))-\sum_i\lambda_i^r d_i^r -\sum_j\lambda_j^c d_j^c.
\end{eqnarray*}
Using $Z(\blambda^r,\blambda^c)=\prod_{i,j}Z(\lambda_i^r,\lambda_j^c)$ and $Z(\lambda_i^r,\lambda_j^c)=1+\exp(\lambda_i^r+\lambda_j^c)$, this gives:
\begin{eqnarray*}
L(\blambda^r,\blambda^c)&=&\sum_{i,j}\log(Z(\lambda_i^r,\lambda_j^c))-\sum_i\lambda_i^r d_i^r -\sum_j\lambda_j^c d_j^c,\\
&=&\sum_{i,j}\log\left(1+\exp(\lambda_i^r+\lambda_j^c)\right)-\sum_i\lambda_i^r d_i^r -\sum_j\lambda_j^c d_j^c.
\end{eqnarray*}

The optimal values of the parameters are easily found using standard methods for unconstrained convex optimization such as Newton's method or (conjugate) gradient descent, possibly with a preconditioner \citep{CG,boyd04}. We will report computational results for two possible choices in Sec.~\ref{experiments}. Gradient descent type methods rely on the gradient of $L$, while Newton's method relies on the gradient as well as the Hessian. Both can easily be computed analytically. The gradient is determined by the first order partial derivatives:
\begin{eqnarray*}
\frac{\partial L} {\partial \lambda_i^r}&=&\sum_j\frac{\exp(\lambda_i^r+\lambda_j^c)}{1+\exp(\lambda_i^r+\lambda_j^c)}-d_i^r,\\
\frac{\partial L} {\partial\lambda_j^c}&=&\sum_i\frac{\exp(\lambda_i^r+\lambda_j^c)}{1+\exp(\lambda_i^r+\lambda_j^c)}-d_j^c.
\end{eqnarray*}
Note that these derivatives have a natural interpretation. Indeed, the sum $\sum_j\frac{\exp(\lambda_i^r+\lambda_j^c)}{1+\exp(\lambda_i^r+\lambda_j^c)}$ is equal to the expected number of ones in the $i$th column for the distribution with the current parameter values, and the partial derivative $\frac{\partial L} {\partial \lambda_i^r}$ is equal to the difference between that expected number and the value $d_i^r$ it needs to be as required by the constraints. The Hessian is determined by the second order partial derivatives, given by:
\begin{eqnarray*}
\frac{\partial^2 L}{\partial\lambda_i^r\lambda_k^r}&=&
\begin{array}{rl}0&\mbox{if } i\neq k\end{array},\\
\frac{\partial^2 L}{\partial\lambda_j^c\lambda_l^c}&=&
\begin{array}{rl}0&\mbox{if } j\neq l\end{array},\\
\frac{\partial^2 L}{\partial\lambda_i^{r2}}&=&
\sum_j\frac{\exp(\lambda_i^r+\lambda_j^c)}
{\left(1+\exp(\lambda_i^r+\lambda_j^c)\right)^2},\\
\frac{\partial^2 L}{\partial\lambda_j^{c2}}&=&
\sum_i\frac{\exp(\lambda_i^r+\lambda_j^c)}
{\left(1+\exp(\lambda_i^r+\lambda_j^c)\right)^2},\\
\frac{\partial^2 L}{\partial\lambda_i^r\lambda_j^c}&=&
\frac{\exp(\lambda_i^r+\lambda_j^c)}
{\left(1+\exp(\lambda_i^r+\lambda_j^c)\right)^2}.
\end{eqnarray*}

The number of Lagrange multipliers to be optimized over, which is crucial for the computational cost of e.g. Newton iterations, is equal to $m+n$. While this is sublinear in the size of the data $mn$, it is still a daunting number for practical sizes of databases. However, the computational and space complexity can often be further reduced, in particular when the numbers of distinct values of $d_i^r$ and of $d_j^c$ are small. Indeed, thanks to symmetry and convexity of $L$, if $d_i^r=d_k^r$ for specific $i$ and $k$ the corresponding optimal values of the Lagrange multipliers $\lambda_i^r$ and $\lambda_k^r$ will be equal as well, and the same goes for the elements of $\blambda^c$. In practice this allows one to drastically reduce the number of free variables, down to the sum of the number of \emph{distinct} expected row sums $d_i^r$ and the number of \emph{distinct} expected column sums $d_j^c$.

Especially for $\cD=\{0,1\}$, almost in all practical cases this allows for a massive reduction in computational complexity. The number of distinct row and column sums can be upper bounded in terms of the dimensions of $\bD$, the number of non-zero elements in $\bD$, and the largest row and column sums in $\bD$, as quantified by the following Lemmas.
\begin{lemma}
In a binary matrix $\bD\in\{0,1\}^{m\times n}$, the number of distinct row sums $\widetilde{m}$ is upper bounded by $\min(m,n+1)$ and the number of distinct column sums $\widetilde{n}$ is upper bounded by  $\min(m+1,n)$.
\end{lemma}
\begin{proof}
Let us prove it for row sums only. Clearly the number of distinct row sums is bounded by the total number of rows $m$. On the other hand, the only possible values of the row sums are $0,1,\ldots,n$, a total of $n+1$ distinct values.
\qed\end{proof}
\begin{lemma}\label{lemma_sqrts}
In a binary matrix $\bD\in\{0,1\}^{m\times n}$ with $\sum_{i,j}\bD(i,j)=s$ (i.e. with $s$ ones), the number of distinct row sums $\widetilde{m}$ is upper bounded by $\sqrt{2s}$, and the same holds for the number of distinct column sums $\widetilde{n}$.
\end{lemma}
\begin{proof}
Let us prove this bound for the number of rows by contradiction. Assume that the number of different row sums is larger than $\sqrt{2s}$. This means that the number of ones in $\bD$ is at least $\sum_{i=0}^{\sqrt{2s}}i=\frac{\sqrt{2s}+\sqrt{2s}^2}{2}=s+\sqrt{\frac{s}{2}}$, if the distinct row sums are $0,1,\ldots,\sqrt{2s}$ and no row sum is equal to another except those equal to $0$. Since $s+\sqrt{\frac{s}{2}}>s$, the assumption is incorrect and the number of different row sums cannot be larger than $\sqrt{2s}$.
\qed\end{proof}
\begin{lemma}
In a binary matrix $\bD\in\{0,1\}^{m\times n}$ with largest row and column sums equal to $d^r_{\mbox{\scriptsize{max}}}$ and $d^c_{\mbox{\scriptsize{max}}}$ respectively, the number of distinct row sums $\widetilde{m}$ and the number of distinct column sums $\widetilde{n}$ are upper bounded by $d^r_{\mbox{\scriptsize{max}}}+1$ and  $d^c_{\mbox{\scriptsize{max}}}+1$ respectively.
\end{lemma}
\begin{proof}
This follows directly from the fact that row and column sums are integers larger than or equal to $0$ and at most equal to $d^r_{\mbox{\scriptsize{max}}}$ and $d^c_{\mbox{\scriptsize{max}}}$ respectively.
\qed\end{proof}
Combining these Lemmas yields the following Theorem.
\begin{theorem}\label{mnbound}
In a binary matrix $\bD\in\{0,1\}^{m\times n}$ with $\sum_{i,j}\bD(i,j)=s$ (i.e. with $s$ ones), and with largest row and column sums equal to $d^r_{\mbox{\scriptsize{max}}}$ and $d^c_{\mbox{\scriptsize{max}}}$ respectively, the following inequalities hold for the number of distinct row sums $\widetilde{m}$ and the number of distinct column sums $\widetilde{n}$:
\begin{eqnarray*}
\widetilde{m}&\leq& \min\left\{m,n+1,\sqrt{2s},d^r_{\mbox{\scriptsize{max}}}+1\right\},\\
\widetilde{n}&\leq& \min\left\{m+1,n,\sqrt{2s},d^c_{\mbox{\scriptsize{max}}}+1\right\}.
\end{eqnarray*}
\end{theorem}
For $\cD$ the set of integers, similar bounds can be obtained.

Let us thus partition the rows in groups of equal value of $d_i^r$, and for the $k$'th group corresponding to expected row sum $\widetilde{d}_k^r$ let us use the Lagrange multiplier $\widetilde\lambda_k^r$, with $\widetilde{m}_k$ denoting the number of rows in that group. Similarly we can define $\widetilde{d}_k^r$, $\widetilde\lambda_l^c$ and $\widetilde{n}_l$. Then we can express the Lagrange dual as:
\begin{eqnarray*}
L(\widetilde\blambda^r,\widetilde\blambda^c)&=&
\sum_{k,l}\widetilde{m}_k \widetilde{n}_l\log \left(1+\exp(\widetilde\lambda_k^r+\widetilde\lambda_l^c)\right)
-\sum_k \widetilde{m}_k\widetilde\lambda_k^r \widetilde{d}_k^r -\sum_l \widetilde{n}_l\widetilde\lambda_l^c \widetilde{d}_l^c.
\end{eqnarray*}
The gradient is easily calculated to be given by:
\begin{eqnarray*}
\frac{\partial L} {\partial \widetilde\lambda_k^r}&=&\sum_l \widetilde{m}_k\widetilde{n}_l\frac{\exp(\widetilde\lambda_k^r+\widetilde\lambda_l^c)}{1+\exp(\widetilde\lambda_k^r+\widetilde\lambda_l^c)}-\widetilde{m}_k\widetilde{d}_k^r,\\
\frac{\partial L} {\partial\widetilde\lambda_l^c}&=&\sum_k \widetilde{m}_k\widetilde{n}_l\frac{\exp(\widetilde\lambda_k^r+\widetilde\lambda_l^c)}{1+\exp(\widetilde\lambda_k^r+\widetilde\lambda_l^c)}-\widetilde{n}_l\widetilde{d}_l^c,
\end{eqnarray*}
and the Hessian elements are computed in a similar way.

The computational complexity of computing the gradient as well as the Hessian are $O(\widetilde{m}\widetilde{n})$. Applying Newton's method then requires solving a linear system of $\widetilde{m}+\widetilde{n}$ equations, with computational complexity $O((\widetilde{m}+\widetilde{n})^3)$ although more efficient approximation methods such as conjugate gradient can be used \citep[e.g.][]{boyd04}. Combining this with Theorem~\ref{mnbound} yields an overall worst-case complexity of at most $O(\sqrt{s^3})$ per iteration for Newton's method and at most $O(s)$ for first order methods such as gradient descent. The space complexity for Newton's method is determined by the size of the Hessian, such that it is bounded by $O(s)$ and thus of the order of the size of the database in sparse representation. For gradient-based or conjugate gradient methods, it is bounded by the size of the gradient $O(\sqrt{s})$. As we will see in the experiments, these results make the MaxEnt approach amenable for practical problems of very large scale.

The reader may be left with one concern: the fact that the MaxEnt model is
a product distribution of independent distributions for each
$\bD(i,j)$ seems to suggest that parameters need to be stored for each of these $m\times n$ element distributions. However, it should be pointed out that one does not need to store the value of $\lambda_i^r+\lambda_j^c$ for each pair of $i$ and $j$. It suffices to store just the $\lambda_i^r$ and
$\lambda_j^c$ to compute the probabilities for any $\bD(i,j)$ in
constant time. Hence, the space required to store the resulting
model is $O(m+n)$, sublinear in the size of the data.

\section{The Invariance of the MaxEnt Matrix Distribution to
$\delta$-Swaps}\label{invariancesection}

The MaxEnt models introduced in Sec.~\ref{mainsection} are explicitly represented probability distributions. As a result, they are useful for defining analytically computable measures of interestingness, as outlined in Sec.~\ref{maxent_for_interestingness}, and we will demonstrate this by designing a concrete interestingness measure in Sec.~\ref{use_case}. Still, it is instructive to point out some relations between our MaxEnt models and the previously proposed swap randomization approaches and generalizations.

\subsection{$\delta$-swaps: a randomization operation on matrices}

First, let us generalize the definition of a swap as follows.
\begin{definition}[$\delta$-swap]
Given an $m\times n$ matrix $\bD$, a $\delta$-swap for rows $i,k$
and columns $j,l$ is the operation that adds a fixed number $\delta$
to $\bD(i,k)$ and $\bD(j,l)$ and subtracts the same number from
$\bD(i,l)$ and $\bD(j,k)$.
\end{definition}

Of course, for a $\delta$-swap to be useful, it must be ensured that
$\bD(i,j)+\delta,\bD(k,j)-\delta,\bD(i,l)-\delta,\bD(k,l)+\delta\in\cD$.
We will refer to such $\delta$-swaps as allowed $\delta$-swaps.
\begin{definition}[Allowed $\delta$-swap]
A $\delta$-swap for rows $i,k$ and columns $j,l$ is said to be
allowed for a given matrix $\bD$ over the domain $\cD$ iff
$\bD(i,j)+\delta,\bD(k,j)-\delta,\bD(i,l)-\delta,\bD(k,l)+\delta\in\cD$.
\end{definition}

Clearly, an allowed $\delta$-swap leaves the row and column sums
invariant. The following Theorem is more interesting.
\begin{theorem}
The probability of a matrix $\bD$ under the MaxEnt distribution
subject to equality constraints on the expected row and column sums
is invariant under allowed $\delta$-swaps applied to $\bD$.
\end{theorem}
\begin{proof}
It is easily verified from Eq.~(\ref{Pindividual}) that:
\begin{eqnarray*}
&&P_{ij}(\bD(i,j))\cdot P_{il}(\bD(i,l))\cdot P_{kj}(\bD(k,j))\cdot P_{kl}(\bD(k,l))\\
&=&P_{ij}(\bD(i,j)+\delta)\cdot P_{il}(\bD(i,l)-\delta)\cdot P_{kj}(\bD(k,j)-\delta)\cdot P_{kl}(\bD(k,l)+\delta)
\end{eqnarray*}
for any $\delta$, rows $i,k$ and columns $j,l$.
\qed\end{proof}

This means that for any $2\times 2$ submatrix of $\bD$, adding a
given number to its diagonal and subtracting the same number from
its off-diagonal elements leaves the total probability of the data under the MaxEnt model invariant.

More generally, the MaxEnt distribution assigns the same probability
to any two matrices that have the same row and column sums. This can
be seen from the fact that Eq.~(\ref{PD1first}) is independent from $\bD$
as soon as the row and column sums $\sum_j\bD(i,j)$ and
$\sum_i\bD(i,j)$ are given. In statistical terms: the row and column
sums are sufficient statistics of the data $\bD$ for the MaxEnt distribution.
We can formalize this in the following Theorem:
\begin{theorem}
The MaxEnt distribution for a matrix $\bD$, conditioned on
constraints on row and column sums of the form
\begin{eqnarray*}
\sum_j\bD(i,j)&=&d^r_i,\\
\sum_i\bD(i,j)&=&d^c_j,
\end{eqnarray*}
denoted as $P(\bD|\sum_j\bD(i,j)=d^r_i,\sum_i\bD(i,j)=d^c_j)$, is
identical to the uniform distribution over all databases satisfying
these constraints.
\end{theorem}
This Theorem further clarifies the connection between the uniform
distribution over all matrices with fixed row and column sums, as
sampled from in \citet{swap07,OVK:08} using swap
randomizations, and the MaxEnt distribution.

\subsection{Special cases of $\delta$-swaps}\label{special_delta_swaps}

The invariants that have been used before in computation intensive
approaches for defining null models for databases are
special cases of these more generally applicable $\delta$-swaps.

For binary databases the condition
$\bD(i,j)+\delta,\bD(k,j)-\delta,\bD(i,l)-\delta,\bD(k,l)+\delta\in\cD$
corresponds to the fact that either $\delta=-1$ and
$\bD(i,k;j,l)=\left(\begin{array}{cc}1&0\\0&1\end{array}\right)$, or
$\delta=1$ and
$\bD(i,k;j,l)=\left(\begin{array}{cc}0&1\\1&0\end{array}\right)$.
Then, the $\delta$-swap is identical to a swap in a binary database.
This shows that the MaxEnt distribution of a binary database is
invariant under swaps as defined in \citet{swap07}. For positive
real-valued databases, the $\delta$-swap operations reduce to the
Addition Mask method in \citet{OVK:08}.

\section{Using the MaxEnt model: Randomizing Databases, and Subjective Interestingness of Tiles}\label{use}

In this Section we will describe how the MaxEnt model from Sec.~\ref{mainsection} allows one to take prior information effectively into account in the data mining process, for concreteness focusing on binary databases. First we show how it can be used to randomize databases highly efficiently, such that it is a fast alternative to swap randomizations. Subsequently, we define a new interestingness measure for tiles in binary databases when contrasted with prior information on the row and column sum, based one of the ideas presented in Sec.~\ref{maxent_for_interestingness}.

\subsection{Randomizing binary databases}\label{randomizing}

The ($\delta$-)swap operations discussed in the previous Section, being simple invariants of the MaxEnt distribution, can be used for randomizing any of the types of databases discussed in this paper. This being said, it should be reiterated that the availability of the MaxEnt distribution should make randomizing the data using $\delta$-swaps unnecessary. Should it be needed to generate randomizations of a given database, one can instead sample directly from the MaxEnt distribution, thus avoiding the computational cost and potential convergence problems faced in randomizing the data. The thus randomized databases can be used exactly as proposed in \citet{swap07} for the assessment of data mining results.

A randomized binary database can be sampled directly from the MaxEnt model by looping through all database entries and sampling a Bernoulli random variable with success probability $P_{ij}(\mathbf{D}(i,j)=1)$. The complexity of this approach is $O(mn)$---prohibitive for large sparse databases. 

Fortunately, a faster approach exists, based on the observation that the number of experiments between two successes in a series of Bernoulli experiments is geometrically distributed. We can sample a sparse representation of a large number of Bernoulli random variables by sampling these so-called inter-arrival-times from the geometric distribution. In this way, the time required is proportional to the number of successes in the set of Bernoulli experiments, rather than to the total number of Bernoulli experiments.

This approach can only be used if all Bernoulli random variables have the same success probability, which is not true for the success probabilities $P_{ij}(\bD(i,j)=1)$ of the entries under the MaxEnt model. However, two database entries will have a different success probability only if either their column sums or their row sums (and thus the associated Lagrange multipliers) are different. From Lemma~\ref{lemma_sqrts}, it is immediate that the number different combinations of row and column sums is $\widetilde{m}\widetilde{n}\leq(\sqrt{2s})^2=2s$, i.e. at most proportional to the number of non-zeros in the original matrix.

Putting things together, this means that a random database can be sampled from the MaxEnt model using a double for-loop over all distinct $\widetilde\lambda_i^r$ and all distinct $\widetilde\lambda_j^c$, with in total at most $2s$ combinations. For each of these combinations, all entries in intersections of the rows and columns with these Lagrange multipliers can be sampled efficiently using the geometric distribution as outlined above.

The total complexity thus consists of two components: sampling the geometric distribution, and the overhead of looping over all combinations of row and column Lagrange multipliers. The latter clearly has a complexity bounded by $s$. The former has a complexity proportional to the number of non-zeros in the sampled matrix, which is proportional to $s$ in expectation and tightly concentrated around it. Hence, the expected complexity is $O(s)$ with the actual complexity tightly concentrated around this.

We should point out that sampling from the MaxEnt model cannot be used as a substitute for swapping if the row and column sums need to be preserved exactly rather than in expectation. This may be the case for categorical data represented by a binary matrix where each column corresponds to an attribute-value. However, in that case a MaxEnt model can be fitted on the categorical representation of the data. Then the constraints will not be on the row and column marginals, but on the number of times each of the attribute values is seen for each of the attributes. Without going into detail, the MaxEnt distribution would then be a product of categorical distributions (one for each database entry), rather than a product Bernoulli distributions.

\subsection{The MaxEnt model to define interestingness of tiles}\label{use_case}

The above shows that the MaxEnt model can be used as an alternative to swap randomizations for the generation of randomized versions of databases. A comparison with swap randomizations for this purpose of randomizing databases can therefore be made, and the empirical results reported in Sec.~\ref{experiments} show that the MaxEnt model allows one to generate randomizations more efficiently than using the swap randomizations strategy.

However, what is more important is that the explicit analytical nature of the MaxEnt model allows one to use it in situations where swap randomizations would be impractical, such as for defining new and subjective measures of interestingness of patterns.

To demonstrate the use of the MaxEnt model for this purpose, we here work out a specific example. In particular, we will focus on binary databases $\bD\in\{0,1\}^{m\times n}$ and a kind of pattern known as a tile \citep{tiles04} that is denoted as $\tau$ and defined as an ordered pair of a set of rows $I\subseteq\{1,\ldots,m\}$ and a set of columns $J\subseteq\{1,\ldots,n\}$, i.e. $\tau=(I,J)$. We say that a tile $\tau=(I,J)$ is present in the database $\bD$, denoted as $\tau\in\bD$, iff $\bD(i,j)=1$ for all $i\in I$ and $j\in J$. Furthermore, we say that the database entry at row $i$ and column $j$ is contained in a tile $\tau=(I,J)$ iff $i\in I$ and $j\in J$, and we denote this more concisely as $(i,j)\in\tau$.

Below we will define a measure of interestingness for tiles and extend these ideas also to sets of tiles. Our approach is based on the second option in Sec.~\ref{maxent_for_interestingness}: computing the compression ratio of information embodied by the statement that a tile is present in the database. In order to quantify this, we need to quantify two things: the \emph{self-information} of a tile pattern with respect to the MaxEnt model representing the prior information, and its \emph{description length} representing its complexity as perceived by the data miner.

\subsubsection{The self-information of a tile}

Let us try to intuitively quantify the amount of information conveyed to a data miner if he is told about the presence of a tile in the database. We argue that it could be formalized by the prior belief the data miner had about the presence or absence of the tile. The most natural way of formalizing this is to use a background distribution representing the data miner's prior expectations, and to compute the probability $Pr(\tau\in\bD)$ of the tile-pattern under this distribution. The smaller $Pr(\tau\in\bD)$, the more information this tile-pattern contains.

A more convenient way of quantifying this is as the negative log-probability, known as the self-information in Shannon's information theory \citep{CoT:92}:
\begin{eqnarray}
\mbox{SelfInformation}(\tau)&=&-\log(Pr(\tau\in\bD)),
\end{eqnarray}
where the probability is taken with respect to the background distribution. If a pattern is more interesting as its probability is smaller, it is equivalently more interesting as its self-information is larger, since minus the logarithm is a monotonically decreasing function.

The self-information is the number of bits (if a base 2 logarithm is used) that is required to encode a particular outcome of a random variable in a Shannon-optimal code. Here that random variable is the indicator variable indicating presence or absence of the tile in the database. Besides its useful interpretation as a code length, the self-information has an important practical advantage over the probability of the presence of the tile as a measure of information content: the logarithm maps extremely small probabilities to numerically more manageable values.

For the MaxEnt model subject to row and column sums, the self-information of a tile pattern can be computed very conveniently by relying on Eqs.~(\ref{eq_independent}-\ref{Pindividual_Bernoulli}):
\begin{eqnarray}\label{self-information}
\mbox{SelfInformation}(\tau)&=&-\sum_{(i,j)\in\tau}\log(p_{ij}),
\end{eqnarray}
where
\begin{eqnarray}\label{pij}
p_{ij}&=&\frac{\exp\left(\lambda^r_i+\lambda^c_j\right)}{1+\exp\left(\lambda^r_i+\lambda^c_j\right)}.
\end{eqnarray}
I.e., the self-information is equal to the sum of the negative log-probabilities that the database entry $\bD(i,j)=1$, summed over all row-column pairs $(i,j)$ in the tile. The fact that it reduces to a simple sum is due to the independence of the database entries $\bD(i,j)$ under the MaxEnt distribution.

\subsubsection{The description length of a tile}

A data miner is never merely interested in receiving as much information as possible. Indeed, the best way to achieve this would be to communicate the entire database to the data miner, which would be of little use. Instead, a data miner will be interested in hearing about patterns that convey new information \emph{as concisely as possible}.

To quantify this, we also need to consider the inherent description length of the pattern, in this case the tile $\tau=(I,J)$, or equivalently its set of rows $I$ and set of columns $J$. A sensible way to describe a set of rows $I$ would be to assume a probabilistic model in which rows occur independently in $I$, each with a certain probability $p$ (see below for more details). Then the description length for the set $I$ under a Shannon-optimal code with respect to this model is given by:
\begin{eqnarray*}
\mbox{DescriptionLength}(I)&=&-\sum_{i\not\in I}\log(1-p)-\sum_{i\in I}\log(p)\\
&=&-(m-|I|)\log(1-p)-|I|\log(p),\\
&=&|I|\log\left(\frac{1-p}{p}\right)+ m\log\left(\frac{1}{1-p}\right),\\
\end{eqnarray*}
Doing the same for the set of columns $J$ and combining both descriptions, the description length for a tile $\tau=(I,J)$ is given by:
\begin{eqnarray}
\mbox{DescriptionLength}(\tau)&=&\mbox{DescriptionLength}(I)+\mbox{DescriptionLength}(J), \nonumber\\
&=&(|I|+|J|)\log\left(\frac{1-p}{p}\right)+(m+n)\log\left(\frac{1}{1-p}\right).
\label{description-length}
\end{eqnarray}

This means that the description length of a tile is equal to its circumference $|I|+|J|$ times a constant, plus another constant term, which makes intuitive sense as a model for the perceived complexity of a tile.

The probability parameter $p$ can be set by the data miner to bias the search toward larger or toward smaller tiles. Indeed, if $p$ is small, the constant component $(m+n)\log\left(\frac{1}{1-p}\right)$ of a tile description length is small while the variable component $(|I|+|J|)\log\left(\frac{1-p}{p}\right)$ is large, thus yielding a short description length for tiles with a small circumference as compared to large ones. In our experiments, we set it equal to the probability that $\bD(i,j)=1$ for $i$ and $j$ sampled uniformly at random, which is equal to the density of the database $p=\frac{1}{mn}\sum_{i,j}\bD$.

\subsubsection{The compression ratio of a tile as interestingness measure}

The interestingness of a tile can now be quantified as the ratio with which information is compressed in the tile pattern:
\begin{eqnarray}\label{compression-ratio}
\mbox{CompressionRatio}(\tau)&=&\frac{\mbox{SelfInformation}(\tau)}{\mbox{DescriptionLength}(\tau)}.
\end{eqnarray}
This ratio expresses the number of bits of information received by the data miner (with respect to the MaxEnt model), per bit received to describe the tile $\tau$. Tiles that have the largest compression ratio are thus most efficient at communicating aspects of the data the data miner did not expect a priori.

\subsubsection{Finding interesting sets of tiles}

It is well-known that the set of individually most interesting patterns is often not the most interesting set of patterns, regardless of which interestingness measure is used (see e.g. \cite{DeZ:07}). This is due to redundancies between the patterns that are individually interesting. So the question arises if we can also use the above tools to define the interestingness of a set of tiles $\mathcal{T}=\{\tau_1,\ldots,\tau_N\}$. It turns out the approach generalizes easily. Furthermore, it yields an additional formal argument in favour of using the ratio of the self-information and description length as interestingness measure for an individual tile.

Describing a set of tiles requires one to describe each of the tiles in the set. Hence, the description length of a set of tiles is quantified by the sum of the description lengths of each of the tiles individually. Slightly overloading notation, for a set of tiles $\mathcal{T}=\{\tau_1,\ldots,\tau_N\}$:
\begin{eqnarray*}
\mbox{DescriptionLength}(\mathcal{T})&=&\sum_{i=1:N}\mbox{DescriptionLength}(\tau_i).
\end{eqnarray*}

The self-information of a set of tiles is generalized as the negative log-probability that all tiles in the set are present in the database. Due to the independence of the database entries under the MaxEnt distribution, this is equal to the sum of the negative log-probabilities that $\bD(i,j)=1$ for all the database entries $(i,j)$ belonging to some tile $\tau\in\mathcal{T}$. Formally:
\begin{eqnarray*}
\mbox{SelfInformation}(\mathcal{T})&=&-\sum_{(i,j):\exists\tau\in\mathcal{T}
\mbox{\scriptsize{ with }} (i,j)\in \tau} \log(p_{ij}),
\end{eqnarray*}
with $p_{ij}$ as defined in Eq.~(\ref{pij}).

In practice, we argue that a data miner has a bounded capacity of taking in and processing patterns. Given this capacity, the data miner would like to receive as much information as possible. In this setting, tile set mining can be formalized by the following optimization problem:
\begin{eqnarray*}
\max_{\mathcal{T}}&&\mbox{SelfInformation}(\mathcal{T}),\\
\mbox{s.t.}&&\mbox{DescriptionLength}(\mathcal{T})\leq u
\end{eqnarray*}
for some upper bound $u$, representing the data miner's capacity.

Interestingly, this problem can be reduced to the \emph{(weighted) budgeted maximum coverage problem} \citep{khuller99}, which is a weighted variant of the maximum set coverage problem. In that problem, a universe of elements is given and with each element a weight is associated. Furthermore, a collection of subsets of the universe is given, each of which has a specified cost. The task is to select a set of subsets from the collection so as to maximize the sum of the weights of the elements in the union of these selected subsets, while respecting an upper bound on the sum of the costs of the selected subsets. 

To reduce our tile mining problem to the budgeted maximum coverage problem, the elements in the universe are the database entries that are equal to 1. The collection of subsets is given by the collection of all tiles present in the database. The weight of the database entry at position $(i,j)$ is equal to the contribution it makes to the information content of a tile containing it, equal to $-\log(p_{ij})$. And the cost of a tile is equal to its description length.

The budgeted maximum coverage problem is a hard combinatorial problem. Fortunately, it can be approximated well by using an efficient greedy algorithm (see \cite{khuller99} for details). The criterion to greedily select a tile $\tau_k$ for inclusion as $k$'th tile in the set is the ratio of the sum of the weights $-\log(p_{ij})$ of database entries $(i,j)\in\tau_k$ not yet contained in earlier selected tiles, versus the description length of $\tau_k$. Formally, with
\begin{eqnarray*}
\mbox{SelfInformation}^+(\tau_k)&=&-\sum_{(i,j)\in\tau_k\mbox{\scriptsize{ and }}(i,j)\not\in\tau_l:l<k}
\log(p_{ij}),
\end{eqnarray*}
in iteration $k$ of the greedy algorithm selects the tile $\tau_k$ maximizing the $\mbox{CompressionRatio}^+(\tau_k)$ defined as follows:
\begin{eqnarray*}
\mbox{CompressionRatio}^+(\tau_k)&=&\frac{\mbox{SelfInformation}^+(\tau_k)}{\mbox{DescriptionLength}(\tau_k)}.
\end{eqnarray*}
In the first iteration, this selection criterion coincides with the interestingness measure $\mbox{CompressionRatio}(\tau)$ as defined earlier, thus corroborating our choice for the this ratio as interestingness measure.

Note that upon selection of a tile $\tau_k$ in iteration $k$, the $\mbox{CompressionRatio}^+(\tau)$ of any other yet unselected tile can only decrease. This can be exploited by the algorithm by keeping all yet unselected tiles in a sorted list, sorted according to their last updated value of $\mbox{SelfInformation}^+(\tau)$. Then, to select the next best tile, the tile at the top of this list is considered and the updated value of its $\mbox{CompressionRatio}^+(\tau)$ is computed. If this value is still larger than the subsequent tile in the list, we can be sure it will remain so even after all subsequent ones are updated too, and it can be selected as the $k+1$'st tile. Otherwise, it must be inserted in the list to keep it sorted, and the second tile in the list is considered. This powerful idea was first introduced in \citet{Min:78}.

The approximation quality of this greedy algorithm is such that any set of $k$ top-ranked tiles in this list has a self-information that is at least $1-\frac{1}{e}$ times the maximum self-information that can be achieved by a set of tiles with a description length that is not longer. Note that this means that the upper bound $u$ on the total description length of the set of tiles does not need to be specified in advance. A data miner can keep querying for the next tile in the list until satisfied, and be sure that all tiles seen so far constitute a tile set that conveys near to the maximum amount of information given its total description length.


\section{Discussion}\label{discussion}

In this Section we point out how the MaxEnt modeling strategy from Sec.~\ref{mainsection} can be used almost directly for modeling network adjacency matrices, and we discuss some relations with the data mining and random networks literature.

\subsection{Networks adjacency matrices as a special case of rectangular databases}

Networks can be represented using their adjacency matrix. A swap operation applied to this adjacency matrix corresponds to swapping a pair of edges $i\rightarrow j$ and $k\rightarrow l$, yielding a new pair of edges $i\rightarrow l$ and $k\rightarrow i$. Such edge swap operations preserve the in- and out-degree of all nodes. They have been introduced and used for the statistical assessment of network patterns, similar to the use of swap randomizations in databases \citep{MSI:02}.

All theory developed this paper for rectangular databases can be applied with minor changes to various types of networks. They can be unweighted or weighted, directed or undirected (using a symmetricity constraint), with and without self-loops (by constraining the diagonal to contain zeros only).

We will not go into further details here. However, because of the importance of networks, in Sec.~\ref{experiments} we will report some empirical results on networks as well.

\subsection{Related literature}

In the Sec.~\ref{prior_si} we discussed prior work on subjective interestingness measures. Here we wish to highlight some further connections with the literature.

\subsubsection{Literature related to the MaxEnt model}

It is instructive to point out how some existing models are related to particular cases of the MaxEnt models introduced in this paper.

Most importantly, the MaxEnt model for binary matrices introduced in
this paper is formally identical to the Rasch model, known from
psychometrics \citep{Rasch61}. This model was introduced to model the
performance of individuals (rows) to questions (columns) in a questionnaire. The matrix
elements indicate which questions were answered correctly or
incorrectly for each individual. The Lagrange multipliers are
interpreted as persons' abilities for the row variables
$\lambda_i^r$, and questions' difficulties $\lambda_j^c$.
Somewhat remarkably, the model was not derived from the MaxEnt principle but
stated directly.

A similar connection exists with the so-called $p^*$ models from
social network analysis \citep{RPK:07}. Although motivated
differently, the $p_1$ model in particular is formally identical to
our MaxEnt model when applied to adjacency matrices for unweighted networks.

Thus, the present paper provides an additional way to look at these
widely used models from psychometrics and social network analysis.
Furthermore, as we have shown, the MaxEnt approach suggests
generalizations, in particular towards non-binary databases and
weighted networks.

When applied to adjacency matrices of networks, the MaxEnt model is related to random network models for
networks with prescribed degree sequences (see \citet{New:03} and
references therein). The most similar model to the ones discussed in
this paper is the one from \citet{ChL:04}. In this paper, the authors
propose to assume that edge occurrences are independent, with each
edge probability proportional to the product of the degrees of the
pair of nodes considered. In the notation of the present paper:
\begin{eqnarray*}
P(\bD)=\prod_{i,j}P(\bD(i,j))&\mbox{with}&
P(\bD(i,j))=\frac{d_id_j}{s},
\end{eqnarray*}
where $s=\sum_i d_i$. Also for this model the constraints on the
expected row and column sums are satisfied.

It would be too easy to simply dismiss this model by stating that
among all distributions satisfying the expected row and column sum
constraints, it is not the maximal entropy one, such that it is
biased in some sense. However, this drawback can be made more
tangible: the model represents a probability distribution only if
$\max_{i,j}d_id_j\leq s$, which is by no means true in all practical
applications, in particular in power-law graphs. This shortcoming is
a symptom of a bias of this model: it disproportionally favors
connections between pairs of nodes both of high degree, such that
for nodes of too high degrees the edge `probability' suggested
becomes larger than $1$. A brief remark considering a similar model
for binary databases was made in \citet{swap07}, where it was
dismissed by the authors on similar grounds.

The uses of the maximum entropy principle in statistics are to numerous to list (\citet{Jaynes57,Jaynes82} are good starting points). Of particular interest to this paper is the prior use of the maximum entropy objective as a regularizer in image reconstruction, and more specifically in computer tomography \citep[e.g.][]{GuS:84}. Here, the intensity distribution over the image is regarded as a probability distribution, to be inferred from integrals of this distribution along various paths. This is similar to the MaxEnt model for binary databases presented in the current paper, where the paths correspond to the rows and columns.

The maximum entropy principle has also been used before in data mining, albeit it for a different purpose and in a different manner than in this paper (e.g. not for the incorporation of prior information into a background model). For example in \citet{Tat:08} the frequency of an itemset was contrasted with an estimate based on the frequency of all its subsets, estimated using the maximum entropy principle. In that paper an itemset was considered more interesting when the actual and estimated frequency differed more strongly, thus defining an objective interestingness measure. In \citet{Sav:04} a similar maximum entropy model had already been used to come up with upper and lower bounds on the possible support values for itemsets. In \citet{PMS:03}, the maximum entropy principle was used for query count approximation on binary databases. Here, the number of results of a query was estimated by relying on probabilistic models of the rows (transactions) in the database. One of the probabilistic models considered in this paper was the maximum entropy model subject to the knowledge of the frequency of the frequent itemsets. For computational reasons, the maximum entropy model here was computed at query time, for just the small subset of variables involved in the query.

\subsubsection{Literature related to the compression ratio as interestingness measure}

The self-information described in Sec.~\ref{use_case} is most strongly related to the surface of a tile \citep{tiles04}. Indeed, if the expected row sums are all equal and similarly for the column sums, $Pr(\bD(i,j)=1)$ under the MaxEnt model is constant throughout the database and the self-information of a tile is simply proportional to its size $|I|\times|J|$. In \cite{tiles04}, each tile was attributed an equal cost as well, and the problem of finding an interesting set of tiles was formalized as finding the set of tiles of a given maximal size that maximizes the number of database entries covered. Then it was observed that solving this optimization problem can be approximated using an (unweighted) budgeted maximum coverage problem.

Hence, our method can be regarded as a refinement of tiling databases in two ways: by giving each database entry a different weight (related to the MaxEnt distribution), and by giving each tile a different cost (depending on its description length). These two modifications make a dramatic difference in the subjective quality of the result, as demonstrated in Sec.~\ref{exp_tiles}.

Another method that is somewhat related is KRIMP \citep{Item06Siebes}, which searches for a set of itemsets that allow one to compress a database. This approach is motivated by the minimum description length principle, regarding data mining essentially as a compression process: a pattern set is considered (objectively) interesting if it has a short description length and simultaneously allows one to describe the data concisely. This leads to objective interestingness measures for patterns. In our approach, we are also searching for a pattern set with a small description length. However, we are not concerned with describing the entire data as concisely as possible, but instead we want to concisely describe \emph{unanticipated aspects of the data} using the pattern set. This makes the resulting interestingness measure subjective in nature.

Lastly, we should note that in a recent conference paper that however did not discuss the MaxEnt model in detail, we introduced a similar measure of interestingness for \emph{noisy} tiles \citep{KoD:10}, which have also been discussed by other authors in different contexts, e.g. in \citet{GMS:04} and \citet{MMG:08}. However, in the current paper we chose to present an interestingness measure for noise-free tiles, as it is sufficient to illustrate the general framework from Sec.~\ref{framework} without risking to overload the reader by the technicalities related to noisy tiles.

\section{Experiments}\label{experiments}

In this Section we assess the computational cost of fitting the MaxEnt model from Sec.~\ref{mainsection} for given expected row and column sums. We also assess the cost and empirical properties of sampling random databases from the MaxEnt model, and we empirically compare this to swap randomizations. Additionally, we will show that the MaxEnt model can be fitted efficiently to very large networks as well. Finally, we will assess the compression ratio as a subjective interestingness measure for tiles by using it for three document databases.

All experiments were done on a 2GHz Pentium Centrino with 1GB Memory, and the code used for each of the experiments will be made freely available on \texttt{www.tijldebie.net/software/maxent}.

\subsection{The MaxEnt model for binary databases}\label{sec_databases}

We report empirical results on ten databases:  seven databases commonly used for evaluation purposes (Retail \citep{retail}, Mushroom, Pumsb, Pumsb star, Connect \citep{Asuncion+Newman:2007}, T10I4D100K, and T40I10D100K \citep{IBMgenerator}), as well as the following three textual datasets turned into databases by considering words as items and documents as transactions:
\begin{description}
\item[ICDM.] All ICDM paper abstracts until 2007. Each abstract is represented by a transaction and words are items, after stop word removal and
stemming.
\item[KDD.] All KDD paper abstracts between 2001 and 2008 (from all sessions)
downloaded from the ACM website. Each abstract is represented by a
transaction and words are items, after stop word removal and
stemming.
\item[Pubmed.] All Pubmed abstracts retrieved by querying with the search
query ``data mining", after stop word removal and stemming.
\end{description}
Some statistics are gathered in Table~\ref{datasets}. The Table
also mentions support thresholds used in some of the experiments reported below, as well as the numbers of closed itemsets satisfying these support thresholds.

For each of these databases we computed the MaxEnt model with expected row and column sums equal to the observed row and column sums in the data.

\begin{table}\caption{Some statistics for the databases investigated: the number of items (columns) and transactions (rows) in the database, the support threshold used in the experiments involving closed itemset mining, the resulting number of closed itemsets, and the average length of each transaction (row) in the databases.}\label{datasets}
\begin{center}
\begin{tabular}{|r|ccccc|}\hline
& \# items & \# trans- & support & \# closed & average \\
& & actions & used & itemsets & transaction \\
& &  &  &  & length\\\hline%
ICDM & 4,976 & 859 & 5 (0.6$\%$) & 365,249 & 48.9 \\%
KDD & 6,154 & 843 & 5 (0.6$\%$) & 2,787,847 & 65.2 \\%
Pubmed & 12,661 & 1,683 & 10 (0.6$\%$) & 1,245,454 & 74.1 \\%
Mushroom & 120 & 8,124 & 81 (1$\%$) & 78,362 & 23.0 \\%
Retail & 16,470 & 88,162 & 9 (0.01$\%$) & 191,088 & 10.3 \\%
Pumsb & 7,117 & 49,046 & 34,332 (70$\%$) & 242,001 & 74.0 \\%
Pumsb star & 7,117 & 49,046 & 14,714 (30$\%$) & 16,486 & 50.5 \\%
Connect & 130 & 67,557 & 40,534 (60$\%$) & 68,349 & 43.0 \\%
T10I4D100K & 1,000 & 100,000 & 100 (0.1$\%$) & 26,962 & 10.0 \\%
T40I10D100K & 1,000 & 100,000 & 1,000 (1$\%$) & 65,236 & 40.0\\\hline%
\end{tabular}
\end{center}
\end{table}

\paragraph{Fitting the MaxEnt model}
The method we used to fit the model is a preconditioned gradient
descent method with Jacobi preconditioner \citep[e.g.][]{CG},
implemented in C++. It is conceivable that more sophisticated
methods will lead to significant further speedups (e.g. methods discussed in \citet{boyd04}, many of which have guaranteed super-linear convergence), but this
one is empirically fast enough and has the advantage of being particularly easy to implement.

To illustrate the speed to compute the MaxEnt distribution,
Fig.~\ref{convergence_analysis} shows plots of the convergence of
the squared norm of the gradient to zero, for the first 25
iterations. The initial value for all Lagrange multipliers was chosen to
be equal to $0$. Noting the logarithmic vertical axis, the
convergence appears exponential. The lower plot in
Fig.~\ref{convergence_analysis} shows the convergence of the Lagrange dual
objective to its minimum over the iterations, a very fast
convergence in just a few iterations.

In all experiments we stopped the iterations as soon as the
normalized squared norm of the gradient became smaller than $10^{-12}$, which is
close to machine accuracy and we believe it is accurate enough for all practical
purposes. The number of iterations required and the overall
computation time are summarized in Table~\ref{computation}.

\begin{table}\caption{The number of iterations required, and the
computation time in seconds to fit the MaxEnt model.}\label{computation}
\begin{center}
\begin{tabular}{|r|cc|}\hline
& \# iterations & time (s)\\\hline%
ICDM & 13 & 0.35 \\%
KDD & 13 & 0.50 \\%
Pubmed & 15 & 1.21 \\%
Mushroom & 35 & 0.012 \\%
Retail & 18 & 2.0 \\%
Pumsb & 36 & 0.10 \\%
Pumsb star & 37 & 0.94 \\%
Connect & 33 & 0.048 \\%
T10I4D100K & 34 & 2.0 \\%
T40I10D100K & 33 & 5.3 \\\hline%
\end{tabular}
\end{center}
\end{table}

\begin{figure}
\begin{center}
\includegraphics[width=1\columnwidth]{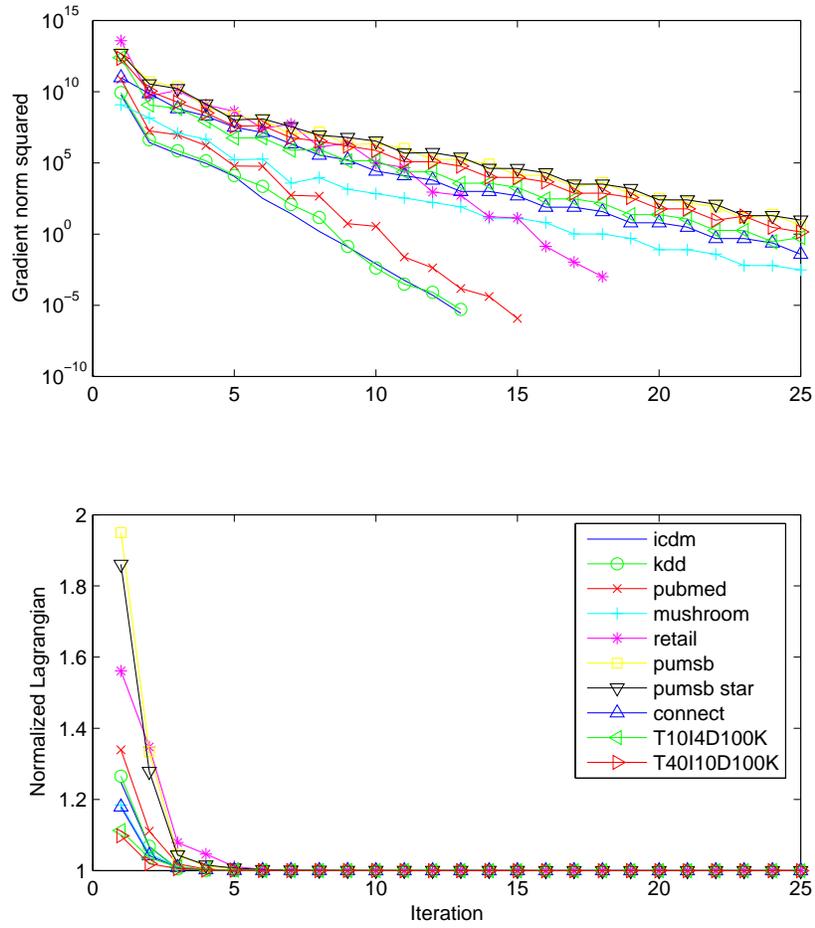}
\caption{Top: the squared norm of the gradient on a logarithmic
scale as a function of the iteration number, plotted for four
databases: KDD abstracts, Mushroom, Pubmed abstracts, and Retail.
This plot shows the exponential decrease of the gradient of the Lagrange dual
optimization problem. In the second plot, the convergence of the
Lagrange dual is shown for the same databases.}\label{convergence_analysis}
\end{center}
\end{figure}

\begin{figure}
\begin{center}
\includegraphics[width=1\columnwidth]{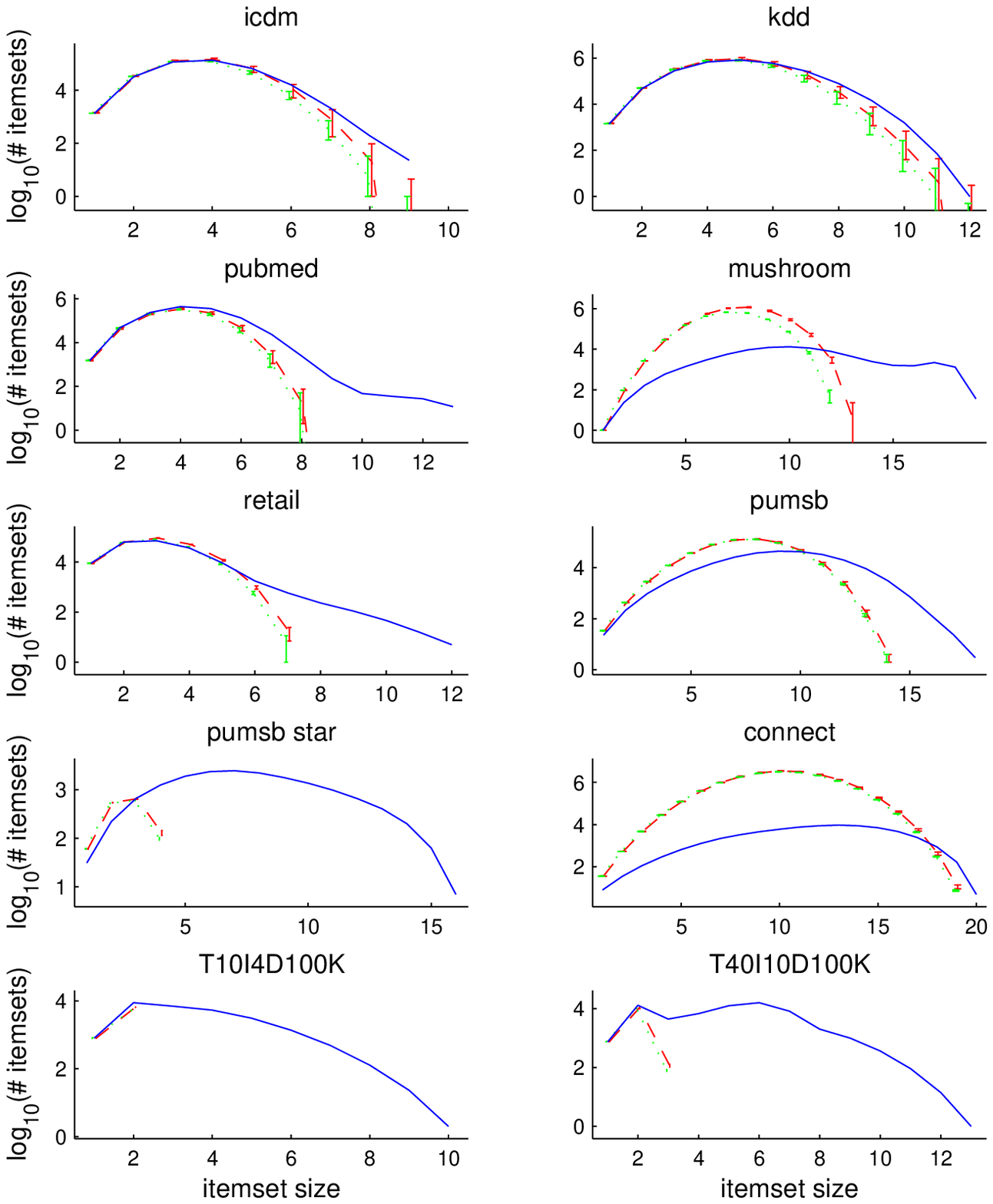}
\caption{For the ten datasets under investigation, these plots show
the number of closed itemsets on a logarithmic scale, as a function
of their size (solid blue line). We also computed the average number of
closed itemsets as a function of their size found on 100 randomized
datasets, along with error bars for the 5th and 95th percentile are shown. The results are plotted both for the swap randomization approach (green dotted line) and the MaxEnt sampling approach (red dashed line).}\label{randomization_testing}
\end{center}
\end{figure}

\paragraph{Assessing data mining results}
Here we illustrate the use of the MaxEnt model for assessing data
mining results in the same spirit as \citet{swap07}.
Figure~\ref{randomization_testing} plots the number of closed
itemsets retrieved on the original data as a function of the itemset size. Additionally, it shows averages with 5th-95th percentile error bars for the results obtained on randomly sampled databases from the MaxEnt model with expected row sums and column sums constrained to be equal to their values on the original data. If desired, one could extract one global measure from these results, as in \citet{swap07}, and compute an empirical p-value by comparing that measure obtained on the actual data with the result on the randomized versions. However, the plots given here do not force one to make such a choice, and they still give a good idea about which itemset sizes are significant in the datasets.

Figure~\ref{randomization_testing} also shows averages and error bars for the results obtained on databases randomized using swaps, done using the code from \citet{swap07} and using five times the number of nonzero database entries (as recommended in \citet{swap07}). It can be noted that the error bars strongly overlap in most cases. The difference between the two randomization strategies is largest for the mushroom dataset. Interestingly, this is a dataset where the transaction sizes are fixed, such that the MaxEnt modeling approach may indeed yield qualitatively different results as compared to the swap randomization approach. Indeed, sampling from the MaxEnt model only conserves the transaction sizes in expectation. A way around this problem is sketched at the end of Sec.~\ref{randomizing}.

\paragraph{Computational cost compared to swap randomizations}
The above shows that randomizing databases using swap randomizations leads to results similar to those obtained by sampling from a fitted MaxEnt model. However, the MaxEnt strategy is five to fifteen times more efficient in generating one randomized database (including the overhead for fitting the MaxEnt model), and about thirty times more efficient when several randomized databases need to be sampled. These computational results are summarized in Fig.~\ref{computation_times}.

\begin{figure}
\begin{center}
\includegraphics[width=1\columnwidth]{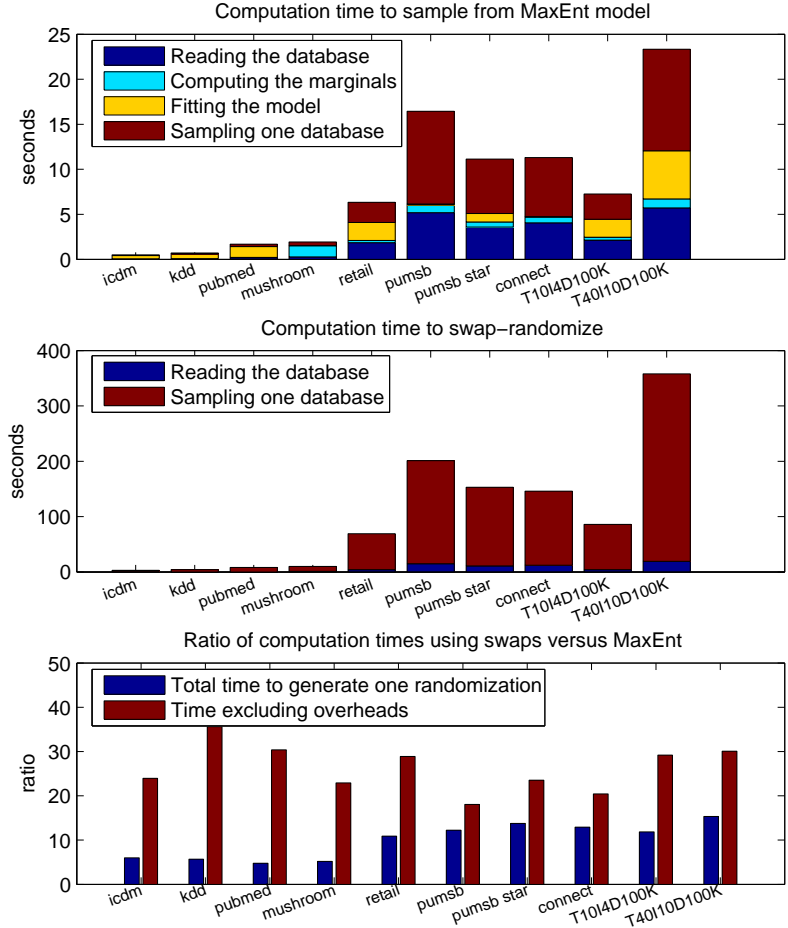}
\caption{Top: The computation times for randomizing the ten databases considered by fitting the MaxEnt model and sampling from this model. The computation time is split up in the time required to fit read the data, compute the marginals, fit the model, and sample one database from the model. Note that in order to sample 100 databases, only the last component would need to be done 100 times. Middle: The computation times for randomizing using swaps, split up into a component for reading the data and a component for generating one database. Bottom: The ratio of the computation times required by both methods, considering the total time to generate one randomization, and considering only the part that needs to be repeated if multiple databases are to be sampled.}\label{computation_times}
\end{center}
\end{figure}

\subsection{The MaxEnt model for various types of networks}

\paragraph{Artificially generated power-law networks}
To assess the feasibility of using the MaxEnt modeling strategy for networks, we artificially generated
power-law (weighted) degree distributions for networks of various
sizes between $n=10$ and $n=10^6$ nodes, with a power-law exponent
of $2.5$. I.e., for each number of nodes $n$ we sampled $n$ expected (weighted)
degrees $d_i$ from the power-law distribution $P(d_i)\sim d_i^{-2.5}$. A
power-law degree distribution with this exponent is often observed
in realistic networks \citep{New:03}, so we believe this is a
representative set of examples. For each of these degree distributions, we fitted four
different types of undirected networks: unweighted networks with and
without self-loops, and positive integer-valued weighted networks
with and without self-loops.

To fit the MaxEnt models for networks we made use of Newton's
method, which we implemented in MATLAB. As can be seen from
Fig.~\ref{time_newton}, the computation time was under 30 seconds
even for the largest network with $10^6$ nodes. The number of Newton
iterations is less than $50$ for all models and degree
distributions considered. \fixme{Note that the increase in
computation time seems linear, as it increases by 3 orders of
magnitude for the last 3 orders of magnitude increase in the network
size. Can something like this even be proven for power-law
distributions??}

\paragraph{Real-life networks}
We also fitted the MaxEnt model to two large real-life networks:
\begin{itemize}
\item A symmetrized snapshot of the internet created by Mark Newman in July 2006.\footnote{Available from \texttt{http://www-personal.umich.edu/$\sim$mejn/netdata/}.} This is an undirected unweighted network. The number of nodes in this network is $22,963$, the total number of different degrees is $161$, with degrees ranging between $1$ and $2390$. Fitting the MaxEnt model required $0.73$ seconds.
\item A network of movie actors \citep{BaA:99}, where edges between actors are weighted by the number of movies in which they jointly appeared. I.e., this is an undirected integer-valued weighted network. The number of nodes (actors) in this network is $127,823$, the total number of unique degrees is $2,526$, with values ranging between $0$ and $8,382$. Fitting the MaxEnt model required $689$ seconds, much larger than for the internet network mostly due to the larger number of unique degrees.
\end{itemize}

\begin{figure}
\begin{center}
\includegraphics[width=0.6\columnwidth]{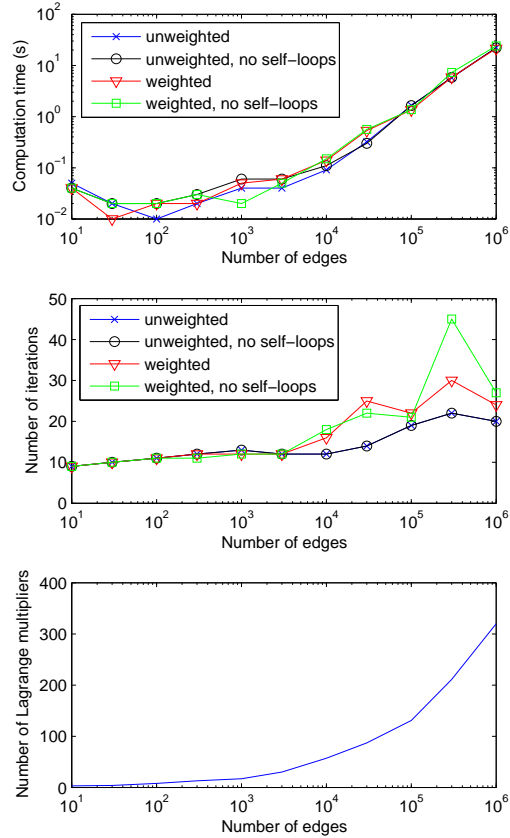}
\end{center}
\caption{Top: The computation time as a function of the size number
of nodes in the network (left). A marker $\times$ is used for the
unweighted model with self-loops, $\circ$ for the unweighted model
without self-loops, $\triangledown$ for the weighted model with
self-loops, and $\Box$ for the weighted model without self-loops.
Note the log-log scale. Middle: The number of iterations required by
the Newton algorithm before convergence. Note the log-scale on the
horizontal axis. Bottom: the number of Lagrange multipliers (i.e.
the number of Lagrange multipliers in the Lagrange dual of the MaxEnt optimization
problem) for the degree sequences investigated, as a function of the
network size. Again, note the log-scale on the horizontal
axis.}\label{time_newton}
\end{figure}

\paragraph{General remarks on the network experiments}
This fast performance can be achieved thanks to the fact that the
number of different degrees observed in the degree distribution is
typically much smaller than the size of the network (see discussion
in Sec.~\ref{sec_lagrange}). The bottom graph in
Fig.~\ref{time_newton}, showing the number of Lagrange multipliers
as a function of the network size supports this. The memory
requirements remain well under control for the same reasons.

It should be pointed out that in the worst case for dense or for
weighted networks (and in particular for real-valued weights), the
number of distinct expected weighted degrees and hence the number of
Lagrange multipliers can be as large as the number of nodes $n$.
This would make it much harder to use off-the-shelf optimization
tools for $n$ much larger than a few thousands. However, the problem can be
made tractable again if it is acceptable to approximate the expected
weighted degrees by grouping subsets of them together into bins, and
replacing their expected degree values by a bin average. In this way the number of
Lagrange multipliers can be reduced to an acceptable level.

\subsection{Using the MaxEnt model to find interesting sets of tiles}\label{exp_tiles}

Here we aim to demonstrate the use of the MaxEnt model in formalizing subjective interestingness where the data miner has prior information on the row and column sums. We applied the method for finding interesting sets of tiles from Sec.~\ref{use_case} to the three abstract datasets discussed above (KDD, ICDM, and Pubmed, see Sec.~\ref{sec_databases}). The particular type of prior information, on the row and the column marginals, makes sense in this setting. Indeed, if words are jointly contained in many documents purely because they are frequent individually, this association is not very meaningful. Similarly, if documents share many words purely because they are long and contain many words, this is not interesting either. Our results below demonstrate that by assuming the frequency of words and lengths of documents as prior information and searching for patterns that contrast with that using our methodology, discovery of such trivial associations is avoided.\footnote{Note that we do not want to sell our method for natural language processing. It simplifies documents to bags of words, disregarding linguistics, and can therefore not be expected to perform similarly. Our sole purpose in applying our method to text is to show its properties, which is hard to do on other databases typically used in this area of research.}

We first mined all tiles covering a number of documents equal to the minimum support threshold mentioned in Table~\ref{datasets} using CHARM \citep{Zaki02CHARM}, and subsequently ran the greedy approximation to the maximum budgeted coverage problem to construct a sorted list of the most interesting of these tiles. As discussed in Sec.~\ref{maxent_for_interestingness}, this sorted list has the property that each set of $k$ top-ranked tiles has close to the maximally achievable self-information given its total description length.

\begin{table}[htb]
\begin{center}
\caption{The left column in this Table shows the sets of words (columns) $J$ and the number of documents $|I|$ containing all these words for the top-15 selected tiles $(I,J)$ by the method described in Sec.~\ref{use_case}, applied to three abstract datasets. The right column gives the results when the tiling databases method from \cite{tiles04} is used.}
\label{tab_results}
\scriptsize
\begin{tabular}{|rl|rl|}
\multicolumn{4}{c}{ICDM}\\\hline
Compression ratio method, $J$ & $|I|$ & Tiling databases, $J$  & $|I|$ \\\hline
classifi machin support vector & 24 & algorithm data & 338 \\
analysi discriminant lda linear & 10 & paper propos & 237 \\
associ database mine rule & 28 & data mine & 279 \\
bayes naiv & 23 & show & 370 \\
algorithm discov frequent mine pattern & 28 & base & 369 \\
nearest neighbor & 20 & result & 359 \\
art state & 22 & approach & 349 \\
cluster data dimensional high subspace & 11 & method & 346 \\
account take & 19 & set & 343 \\
play role & 14 & problem & 330 \\
document text word & 14 & present & 305 \\
exampl learn train & 17 & perform & 265 \\ 
algorithm em expect maximization & 8 & model & 239 \\ 
frequent item itemset mine & 18 & larg & 221 \\
classifi decis tree & 20 & algorithm propos & 271 \\
\hline\multicolumn{4}{c}{}\\
\multicolumn{4}{c}{KDD}\\\hline
Compression ratio method, $J$ & $|I|$ & Tiling databases, $J$ & $|I|$ \\\hline
machin support svm vector & 25 & data paper & 389 \\
art state & 39 & algorithm propos & 246 \\
labeled learn supervised unlabeled & 10 & data mine & 312 \\
associ mine rule & 36 & base method & 202 \\
express gene & 25 & result show & 196 \\
frequent itemset & 28 & problem & 373 \\
graph larg network social & 15 & data set & 279 \\ 
column row & 13 & approach & 330 \\
algorithm faster magnitud order & 12 & model & 301 \\
algorithm data paper propos real synthetic & 27 & present & 296 \\
answer question  & 18 & larg & 286 \\
nearest neighbor  & 13 & applic & 271 \\
classifi featur machin support text vector  & 9 & perform & 266 \\
precis recal  & 14 & real & 255 \\
decis tree  & 33 & inform & 240 \\
\hline\multicolumn{4}{c}{}\\
\multicolumn{4}{c}{Pubmed}\\\hline
Compression ratio method, $J$ & $|I|$ & Tiling databases, $J$ & $|I|$ \\\hline
algorithm data database drug gamma item mgps \ldots & 10 & data mine & 1545 \\
\ldots mine multi poisson report safeti shrinker system &  &  &  \\
chain data express gene polymerase reaction revers & 14 & data method result  & 438 \\
chromatography liquid mass ms spectrometry & 11 & analysi data & 754 \\
est express sequenc tag & 37 & express gene & 369 \\
advers data detect drug mine reactions \ldots  & 11 & base & 713 \\
\ldots report signal spontan  &  &  &  \\
acid amino data mine protein sequenc  & 32 & studi & 638 \\
nucleotide polymorphisms singl snps studi  & 13 & data inform & 577 \\
availability data motivation result this$\_$website  & 35 & develop & 594 \\
artifici network neural  & 36 & database & 570 \\
data express gene pcr rt  & 29 & system & 511 \\
data ii iii iv  & 16 & approach & 510 \\
data database interface user web  & 27 & tool & 497 \\
arabidopsis express gene plant  & 25 & set & 486 \\
care decis health support system  & 15 & high & 469 \\
data map organizing som  & 13 & present & 431 \\
\hline
\end{tabular}
\end{center}
\end{table}

The top-15 tiles as returned by our method are summarized in the left column of Table~\ref{tab_results}. (Note that the choice to show just 15 tiles is arbitrary.) We report the sets of words sorted alphabetically, as well as the size of the set of documents. For comparison, the results of the related method from \cite{tiles04} are shown in the right column.

We argue that our method achieves the most sensible results in terms of non-redundancy and interestingness of the highly ranked tiles, many of them coinciding with major topics and concepts in data mining (ICDM, KDD) and data mining applied to biological problems (Pubmed). In contrast, the tile-based method seems to favor tiles with few but individually frequent items.

Figure~\ref{top-50} shows the compression ratio for the 50 top-ranked tiles in each of the three abstract databases. For reference, we also ran our method on a randomized version of each of these three databases (randomized as described in Sec.~\ref{randomizing}) and we show the resulting compression ratios on the same plots. Although the compression ratio exceeds 1 only for a few top-ranked tiles on the actual databases, it is higher than the largest compression ratio seen on the randomized databases at least for all of the 50 top-ranked tiles for which data is shown in these plots. This corroborates that the top-ranked tiles are indeed meaningful and cannot be attributed to randomness.

\begin{figure}[htb]
\begin{center}
\includegraphics[width=\textwidth]{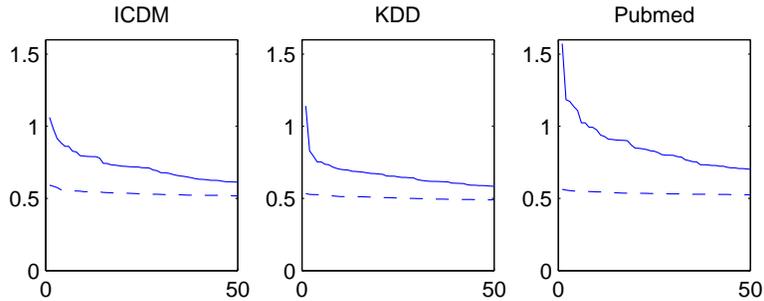}
\caption{The compression ratio of the top-50 ranked tiles as a function of their rank, evaluated on the three abstract databases (full line) as well as on a randomized version of these databases (dashed lines).}\label{top-50}
\end{center}
\end{figure}

We believe that the true assessment of a subjective measure should be the subjective one provided just above. Still, we also conducted an objective comparison between our interestingness measure and the tiling method from \cite{tiles04}. We took each of the 3 abstract databases and artificially embedded large tiles in these databases by appending $k$ rows (`documents') as well as $k$ columns (`words') to it, ensuring that the database cells in the intersection of the $k$ new rows and the $k$ new columns are all equal to $1$. To ensure that the overall statistics of the database are maintained, we furthermore randomly added ones in the cells in the intersection of the existing columns and the new rows, so as to ensure that the column densities remain the same. Similarly, we randomly added ones to the cells in the intersection of the existing rows and new columns to ensure the row frequencies are unchanged. We did this for various tile sizes $k$, namely for $k=5,10,15,$ and $20$.

We then computed a ranking of tiles based on our newly proposed compression ratio interestingness measure, as well as based on the surface of a tile \citep{tiles04}. We then compared the rank of the embedded tile (or if possible a larger one containing it) in the ranking returned by these two interestingness measures. The results, shown in Table~\ref{rankings}, clearly demonstrate that our proposed approach is much more effective in finding the embedded tile.

\begin{table}[htb]
\begin{center}\caption{The ranks of the embedded tile for varying size $k$ in the rankings returned by our compression ratio method and the tiling method from \citet{tiles04}. Note that for tile size $5$, no results are available for the Pubmed database as the support threshold there is $10$ such that the embedded tile is not retrieved. Clearly, the newly proposed method is much more effective at ranking the embedded tile highly.}\label{rankings}
\begin{tabular}{|r|r|c|c|c|c|}\hline
 & $k$ & 5 & 10 & 15 & 20 \\\hline
ICDM & Compression ratio method & 3 & 1 & 1 & 1 \\
& Tiling databases & $>100$ & 71 & 16 & 3 \\\hline
KDD & Compression ratio method & 2 & 1 & 1 & 1 \\
& Tiling databases & $>100$ & 94 & 19 & 5 \\\hline
Pubmed & Compression ratio method & N/A & 1 & 1 & 1 \\
& Tiling databases & N/A & $>100$ & 76 & 19 \\\hline
\end{tabular}
\end{center}
\end{table}

\section{Conclusions}

A significant amount of data mining research has been devoted to the assessment of data mining results when contrasted to prior information, leading to the notion of subjective measures of interestingness. A key task in this endeavour is the formalization of prior information.

In this paper, we have introduced a new modeling approach for prior information based on the maximum entropy principle. Fitting the resulting MaxEnt distribution boils down to a well-posed convex optimization problem. We have also outlined various ways in which the MaxEnt model can be used to contrast patterns in data with prior information, in order to come up with subjective interestingness measures.

Applying this general framework to rectangular databases and prior information on the row and column sums, it turns out that the MaxEnt model can be represented particularly compactly, and specific properties can be exploited to dramatically enhance computational efficiency. Furthermore, we showed how the MaxEnt model can be used efficiently to sample random databases satisfying the prior information. Finally, we also worked out the details of a new interestingness measure for tiles, referred to as the compression ratio, that takes account of row and column sums as prior information.

In our further work we will investigate other interesting use cases of the general framework laid out in Sec.~\ref{framework}. For example, in line with the alternative randomization strategies suggested in \citet{HOV:09}, we will investigate other types of prior information on rectangular databases, such as on the variance within rows or columns (then the domain of the database elements can be chosen to be $\cD=\mathbb{R}$ and the resulting MaxEnt distribution would be a product of Gaussian distributions), the support of certain itemsets (for which also the ideas from \citet{PMS:03} and \citet{Cal:08} will prove useful), or an existing cluster structure in the data. Furthermore, the strategy can be applied to other data types as well, such as categorical data (see end of Sec.~\ref{randomizing}). In parallel, we will investigate the use of the MaxEnt modeling strategy for other types of data such as relational databases.


\subsubsection*{acknowledgements}

This work is supported by the EPSRC grant EP/G056447/1. The author is grateful to Bart Goethals, Nello Cristianini, Akis Kontonasios, and Eirini Spyropoulou for interesting discussions relating to this work.

\bibliographystyle{spbasic}
\bibliography{biblio}

\end{document}